\documentclass[twoside]{article}

\usepackage[accepted]{aistats2017}
\usepackage{natbib,url,amsmath,sectsty,amstext,relsize,floatpag,wrapfig,listings,subfigure,amsfonts,amssymb,amsthm,mathabx,hyperref,grffile,colortbl,dsfont,bbm,setspace, overpic, floatrow, stmaryrd, tabularx}
\usepackage[shortlabels]{enumitem} 

\usepackage[linewidth=1pt]{mdframed}

\allowdisplaybreaks 

\newcommand{\beginsupplement}{  
        \setcounter{section}{0}
        \renewcommand{\thesection}{S\arabic{section}} %
         \renewcommand{\thesubsection}{\thesection.\arabic{subsection}}
        \setcounter{table}{0}
        \renewcommand{\thetable}{S\arabic{table}} %
        \setcounter{figure}{0}
        \renewcommand{\thefigure}{S\arabic{figure}} %
     }
\usepackage{multibib} 
\newcites{si}{Additional References for the Supplementary Information}

\newcounter{assumption}

\DeclareMathOperator*{\argmax}{argmax}
\DeclareMathOperator*{\argmin}{argmin}
\newcommand\independent{\protect\mathpalette{\protect\independenT}{\perp}}
\def\independenT#1#2{\mathrel{\rlap{$#1#2$}\mkern2mu{#1#2}}}

\definecolor{LightGray}{gray}{0.9}
\newcommand\numberthis{\addtocounter{equation}{1}\tag{\theequation}}

\newtheorem{thm}{Theorem}
\newtheorem{lem}{Lemma}
\newcounter{factnum}
\newcounter{claimnum}
 \newcounter{defnum}
\begin{document}

%

%

\twocolumn[

\aistatstitle{Learning Optimal Interventions}

\aistatsauthor{ Jonas Mueller \And David N. Reshef \And George Du \And Tommi Jaakkola }
\aistatsaddress{ MIT Computer Science and Artificial Intelligence Laboratory } ]

\begin{abstract} 
Our goal is to identify beneficial interventions from observational data.  We consider interventions that are narrowly focused (impacting few covariates) and may be tailored to each individual or globally enacted over a population.  For applications where harmful intervention is drastically worse than proposing no change, we propose a conservative definition of the optimal intervention.  Assuming the underlying relationship remains invariant under intervention, we develop  efficient algorithms to identify the optimal intervention policy from limited data and provide theoretical guarantees for our approach in a Gaussian Process setting.  Although our methods assume covariates can be precisely adjusted, they remain capable of improving outcomes in misspecified settings where interventions incur unintentional downstream effects.  Empirically, our approach identifies good interventions in two practical applications: gene perturbation and writing improvement.  
\end{abstract}

\section{Introduction}
\label{sec:intro}
In many data-driven applications, including medicine, the primary interest is identifying interventions that produce a desired change in some associated outcome.  Due to experimental limitations, learning in such domains is commonly restricted to an observational dataset $\mathcal{D}_n := \left\{ \left(x^{(i)}, y^{(i)}\right) \right\}_{i=1}^n$ which consists of IID samples from a population with joint distribution $\mathbb{P}_{XY}$ over covariates (features)  $X \in \mathbb{R}^d$ and outcomes $Y \in \mathbb{R}$.  Typically, such data is analyzed using models which facilitate understanding of the relations between variables (eg.\ assuming linearity/additivity).  Based on conclusions drawn from this analysis, the analyst decides how to intervene in a manner they confidently believe will improve outcomes.  

Formalizing such beliefs via Bayesian inference, we develop a framework that identifies beneficial interventions directly from the data.  In our setup, an intervention on an individual with pre-treatment covariates $X$  produces post-treatment covariate values $\widetilde{X}$ that determine the resulting outcome $Y$ (depicted as the graphical model: $X \rightarrow \widetilde{X} \rightarrow Y)$.  Each possible intervention results in a diffferent $\widetilde{X}$.  More concretely, we make the following simplifying assumption: 
\begin{equation}
Y = f(\widetilde{X}) + \varepsilon \  \text{ with } \mathbb{E}[\varepsilon] = 0, \varepsilon \independent  \widetilde{X}, X 
\label{eq:functionassumption}
\end{equation}
for some underlying function $f$ that encodes the effects of causal mechanisms (ie.\ $\widetilde{X}$ represents a fair description of the system state, and some covariates in $\widetilde{X}$ causally affect $Y$, not vice-versa).  The observed data is comprised of naturally occurring covariate values where we presume $\widetilde{x}^{(i)} = x^{(i)}$ for  $i=1,\dots,n$ (ie.\ the system state remains static without intervention, so the observed covariate values directly influence the observed outcomes).  Moreover, we assume the relationship between these covariate values and the outcomes remains invariant, following the same (unknown) function $f$ for any $\widetilde{X}$ arising from one of our feasible interventions (or no intervention at all).  Note that this assumption precludes the presence of hidden confounding.  \cite{Peters2015} have also relied on this invariance assumption, verifying it as a reasonable property of causal mechanisms in nature.

Given this data, we aim to learn an intervention policy defined by a covariate transformation $T : \mathbb{R}^d \rightarrow \mathbb{R}^d$, applied to each individual in the population.  Here, $T(x)$ presents a desired setting of the covariates that should be reflected by subsequent intervention to actually influence outcomes.  When $T$ only specifies changes to a subset of the covariates, an intervention seeking to realize $T$ may have unintended side-effects on covariates outside of this subset.  We ignore such ``fat hand'' settings  \mbox{\citep{Duvenaud2010}} until \S\ref{causalconnection}.  Instead, our methods assume interventions can always be carried out with great precision to ensure the desired transformation $T$ is exactly reflected in the post-treatment values: $\widetilde{x} = T(x)$.  Our goal is to identify the transformation $T$ which produces the largest corresponding post-treatment improvement with high certainty.  $T(x)$ can either represent a  single mapping to be performed on all individuals (global policy) or encode a personalized policy where the intervened upon variables and their values may change with $x$.   

Our strong assumptions are made to ensure that statistical modeling alone suffices to identify beneficial interventions.  While many real-world tasks violate these conditions, there exist important domains in which  violations are sufficiently minor that our methods can discover effective interventions (cf.\ \citet{Carulla2016, Peters2015}).  We use two applications to illustrate our framework.  One  is a writing improvement task where the data consists of documents labeled with associated outcomes (eg.\ grades or popularity) and the goal is to suggest beneficial changes to the author. Our second example is a gene perturbation task where the expression of some regulatory genes can be up/down-regulated in a population (eg.\ cells or bacteria) with the goal of inducing a particular phenotype or activation/repression of a downstream gene.  In these examples, covariates are known to cause outcomes and our other assumptions may hold to some degree,  depending on the type of external intervention used to alter covariate values.

The contributions of this work include: (1) a formal definition of the optimal intervention that exhibits  desirable characteristics under uncertainty due to limited data, (2) widely applicable types of (sparse) intervention policy that are easily enacted across a whole population, (3) algorithms to find the optimal intervention under practical constraints, (4) theoretical insight regarding our methods' properties in Gaussian Process settings as well as certain misspecified applications.

\section{Related Work}
\label{sec:relwork}

The same invariance assumption has been exploited by \citet{Peters2015} and \citet{Carulla2016} for causal variable selection in regression models.  Recently, researchers such as \citet{Duvenaud2010} and \citet{Kleinberg2015} have supported a greater role for predictive modeling in various decision-making settings.  \cite{Zeevi2015} use gradient boosting to predict glycemic response based on diet (and personal/microbiome covariates), and found they can naively leverage their regressor to select personalized diets which result in superior glucose levels than the meals proposed by a clinical dietitian.  As treatment-selection in high-impact applications (eg.\ healthcare)  grows increasingly reliant on supervised learning methods, it is imperative to properly handle uncertainty.  

Nonlinear Bayesian predictive models have been employed by  \citet{Hill2012}, \citet{Brodersen2015}, and \citet{Krishnan2015} for quantifying the effects of a given treatment from observations of individuals who have been treated and those who have not.  Rather than considering a single given intervention, we introduce the notion of an optimal intervention under various practical constraints, and how to identify such a policy from a limited dataset (in which no individuals have necessarily received any interventions).

Although our goals appear similar to Bayesian optimization and bandit problems \citep{Shahriari2016, Agarwal2013}, additional data is not collected in our setup.  Since we consider settings where interventions are proposed based on all available data, acquisition functions for sequential exploration of the response-surface are not appropriate.  As most existing data is not generated through sequential experimentation, our methods are more broadly applicable than iterative approaches like Bayesian optimization.  

A greater distinction is our work's focus on the pre vs.\ post-intervention change in outcome for each particular individual, whereas Bayesian optimization seeks a single globally optimal configuration of covariates.  In practice, feasible covariate transformations are constrained based on an individual's naturally occurring covariate-values, which stem from some underlying population beyond our control.  For example in the writing improvement task, the goal is not to identify a globally optimal configuration of covariates that all texts should strive to achieve, but rather to inform a particular author of simple modifications likely to improve the outcome of his/her existing article.  Appropriately treating such constraints is particularly important when we wish to prescribe a global policy corresponding to a single  intervention applied to all individuals from the population (there is no notion of an underlying population in Bayesian optimization).

\section{Methods}
\label{sec:methods}
Our strategy is to first fit a Bayesian model for $Y \mid X$ whose posterior encodes our beliefs about the underlying function $f$ given the observed data.  Subsequently, the posterior for $f \mid \mathcal{D}_n$ is used to identify a transformation of the covariates $T : \mathbb{R}^d \rightarrow \mathbb{R}^d$ which is likely to improve expected post-intervention outcomes according to our current beliefs.  The posterior for $f \mid \mathcal{D}_n$ may be summarized at any points $x, x' \in \mathbb{R}^d$ by mean function $\mathbb{E}[f(x) \mid  \mathcal{D}_n]$ and covariance function $\text{Cov}(f(x) , f(x') \mid   \mathcal{D}_n)$.

\subsection{Intervening at the Individual Level}
For $x \in \mathbb{R}^d$ representing the covariate-measurements from an individual, we are given a set $\mathcal{C}_x \subset \mathbb{R}^d$ that denotes constraints of possible transformations of $x$. Let $T(x) = \widetilde{x} \in \mathcal{C}_x$ denote the new covariate-measurements of this individual after a particular intervention on $x$ which alters covariates as specified by transformation $T: \mathbb{R}^d  \rightarrow \mathbb{R}^d$.  Recall that we assume an intervention can be conducted to produce post-treatment covariate-values that exactly match any feasible transformation: $\widetilde{x} = T(x)$, and we thus write $f(T(x))$ in place of $\mathbb{E}_{\varepsilon}[Y \mid \widetilde{X} = T(x)]$.  

We first consider  \emph{personalized interventions} in which $T$ may be tailored to a particular  $x$.
Under the Bayesian perspective, $f \mid \mathcal{D}_n$ is randomly distributed according to our posterior beliefs, and we define the \emph{individual expected gain} function:
\begin{equation} G_x(T) :=  f(T(x)) - f(x) \mid \mathcal{D}_n 
\label{eq:indgain}
\end{equation} 
Since $f(x) = \mathbb{E}_\varepsilon [Y \mid \widetilde{X} = x]$,  random function $G_x$ evaluates the expected outcome-difference at the post vs. pre-intervention setting of the covariates  (this expectation is over the noise $\varepsilon$, not our posterior).  To infer the best personalized intervention (assuming higher outcomes are desired), we use optimization over vectors $T(x) \in \mathbb{R}^d$ to find:
\begin{equation}
T^*(x) = \argmax_{T(x) \in \mathcal{C}_x} \ F^{-1}_{G_x(T)}(
\alpha) 
\label{eq:personalized}
\end{equation}
where $F^{-1}_{G(\cdot)}(\alpha)$ denotes the $\alpha^{\text{th}}$ quantile of our posterior distribution over $G(\cdot)$.  We choose $0 < \alpha < 0.5$, which implies the intervention that produces $T^*(x)$ should improve the expected outcome with probability $\ge 1-\alpha$ under our posterior beliefs. 

Defined based on known constraints of feasible interventions, the set $\mathcal{C}_x \subset \mathbb{R}^d$ enumerates possible transformations that can be applied to an individual with covariate values $x$.  If the set of possible interventions is independent of $x$ (ie.\ $\mathcal{C}_x = \mathcal{C} \ \forall x$), then our goal is similar to the optimal covariate-configuration problem studied in Bayesian optimization.  However, in many practical applications, $x$-independent transformations are not realizable through intervention.  Consider gene perturbation, a scenario where it is impractical to simultaneously target more than a few genes due to technological limitations.  If alternatively intervening on a quantity like caloric intake, it is only realistic to change an individual's current value by at most a small amount.  The choice  $\mathcal{C}_x := \{ z \in \mathbb{R}^d : ||x - z||_0 \le k \}$ reflects the constraint that at most $k$ covariates can be intervened upon.  We can denote limits on the amount that the $s^{\text{th}}$ covariate may be altered by $\mathcal{C}_x := \{ z \in \mathbb{R}^d : |x_s - z_s| \le \gamma_s \}$ for $s \in \{1,\dots,d\}$.  In realistic settings, $\mathcal{C}_x$ may be the intersection of many such sets reflecting other possible constraints such as boundedness, impossible joint configurations of multiple covariates, etc. 

For any $x, T(x) \in \mathbb{R}^d$: the posterior distribution for $G_x(T)$ has: 
\begin{align*} & \text{mean}  = \mathbb{E}[f(T(x) \mid \mathcal{D}_n ] - \mathbb{E}[f(x) \mid \mathcal{D}_n \numberthis \\
 & \text{variance} = \text{Var}(f(T(x)) \mid \mathcal{D}_n) + \text{Var}(f(x) \mid \mathcal{D}_n)  \\
& \hspace*{18mm} - 2\text{Cov}(f(T(x)), f(x) \mid \mathcal{D}_n)  \numberthis
 \end{align*}
 which is easily computed using the corresponding mean/covariance functions of the posterior $f \mid \mathcal{D}_n$.  When $T(x) = x$, the objective in (\ref{eq:personalized}) takes value 0, so any superior optimum corresponds to an intervention we are confident will lead to expected improvement.  If there is no good intervention in $\mathcal{C}_x$ (corresponding to a large increase in the posterior mean) or too much uncertainty about $f(x)$ given limited data, then our method simply returns $T^*(x) = x$ indicating  no intervention should be performed.  
 
Our objective exhibits these desirable characteristics because it relies on the posterior beliefs regarding both $f(T(x))$ and $f(x)$, which are tied via the covariance function.  In contrast, a similarly-conservative lower confidence bound objective (ie.\ the UCB acquisition function with lower rather than upper quantiles)  would only consider  $f(T(x))$, and could  propose unsatisfactory transformations where ${{\mathbb{E}[f(x) \mid \mathcal{D}_n ]} > \mathbb{E}[f(T(x)) \mid \mathcal{D}_n ]}$.

\subsection{Intervening on Entire Populations}
The above discussion focused on personalized interventions tailored on an individual basis.  In certain applications, policy-makers are interested in designing a single intervention which will be applied to all individuals from the same underlying population as the data.  Relying on such a \emph{global policy} is the only option in cases where we no longer observe covariate-measurements of new individuals outside the data.   In our gene perturbation example, gene expression may no longer be individually profiled in future specimens that receive the decided-upon intervention to save costs/labor.  

Here, the covariates $X$  are assumed distributed according to some underlying (pre-intervention) population, and we define the \emph{population expected gain} function:
\begin{equation*} G_X(T) := \mathbb{E}_X [ G_x(T) ] =  \mathbb{E}_X \big[ f(T(x)) - f(x) \mid \mathcal{D}_n \big] 
\end{equation*}
which is also randomly distributed based on our posterior ($\mathbb{E}_X$ is expectation with respect to the covariate-distribution $X$ which is not modeled by $f \mid \mathcal{D}_n$).  Our goal is now to find a single transformation $T : \mathbb{R}^d \rightarrow  \mathbb{R}^d$ corresponding to a \emph{population intervention} which will (with high certainty under our posterior beliefs) lead to large outcome improvements on average across the population:  
\begin{equation}
T^* = \argmax_{T \in \mathcal{T}} \ F^{-1}_{G_X(T)}(\alpha) 
\label{eq:policy}
\end{equation}
Here, the family of possible transformations $\mathcal{T}$ is constrained such that $T(x) \in \mathcal{C}_x$ for all $T \in \mathcal{T}, x \in \mathbb{R}^d$.  As a good model of our multivariate features may be {unknown}, we instead work with the empirical estimate:
\begin{flalign*}
&& T^* & = \argmax_{T \in \mathcal{T}} \ F^{-1}_{G_n(T)}
(\alpha) && \numberthis
\label{eq:empiricalpolicy} 
\\
\text{where }  &&  G_n(T) & := \frac{1}{n}  \hspace*{-0.1mm} \sum_{i=1}^n \big[ f(T(x^{(i)})) - f(x^{(i)}) \big] \ \mid \mathcal{D}_n &&
\end{flalign*}
is the \emph{empirical} population expected gain, whose posterior distribution has:
\begin{align*}
& \hspace*{-2mm} \text{mean} = \frac{1}{n} \hspace*{-0.5mm} \sum_{i=1}^n \mathbb{E}[f(T(x^{(i)})) \mid \mathcal{D}_n] - \mathbb{E}[f(x^{(i)}) \mid \mathcal{D}_n]  \numberthis \\
& \hspace*{-2mm} \text{variance} =  \frac{1}{n^2} \sum_{i=1}^n \sum_{j=1}^n \Big[ \text{Cov}\left( f(x^{(i)}), f(x^{(j)}) \mid \mathcal{D}_n \right) \\
& \hspace*{15.5mm} - \text{Cov}\big( f(T(x^{(i)})), f(x^{(j)}) \mid \mathcal{D}_n \big) \\
& \hspace*{15.5mm} - \text{Cov}\big( f(x^{(i)}), f(T(x^{(j)})) \mid \mathcal{D}_n \big) \\[-5pt]
&  \hspace*{15.5mm} + \text{Cov}\big(f(T(x^{(i)})), f(T(x^{(j)})) \mid \mathcal{D}_n \big) \Big] \numberthis
\end{align*}

The population intervention objective in (\ref{eq:empiricalpolicy}) is again 0 for the identity mapping $T(x) = x$.  Under excessive uncertainty or a dearth of beneficial transformations in $\mathcal{T}$, the policy produced by this method will again simply be to perform no intervention.   In this population intervention setting, $T$ is designed  assuming future individuals will stem from the same underlying distribution as the samples in $\mathcal{D}_n$.    Although $T$ is a function of $x$, the form of the transformation must be agnostic to  the specific values of $x$ (so the intervention can be applied to new individuals without measuring their covariates).  

We consider two types of transformations that we find widely applicable.   \emph{Shift} interventions involve transformations of the form: $T(x) = x + \Delta$ where $\Delta \in \mathbb{R}^d$ represents a (sparse) shift that the policy applies to each individuals' covariates (eg.\ always adding 3 to the value of the second covariate corresponds to $T(x) = [x_1, x_2 + 3, \dots, x_d]$).  \emph{Covariate-fixing} interventions are policies which set certain covariates to a constant value for all individuals, and involve transformations $T_{\mathcal{I}\shortrightarrow z}(x) = [z_1, \dots, z_d]$ such that for some covariate-subset $\mathcal{I} \subseteq \{1,\dots,d\}:  z_j = x_j \ \forall j \notin \mathcal{I}$ and for $j \in \mathcal{I}$: $z_j \in \mathbb{R}$ is  fixed across all $x$  (eg.\ always setting the first covariate to 0, for example in gene knockout, corresponds to $T(x) = [0, x_2, \dots, x_d]$ $\forall x$).  Figure \ref{fig:examples} depicts examples of these different interventions.  Under a sparsity constraint, we must carefully model the underlying population in order to identify the best covariate-fixing intervention (here, setting $X_1$ to a large value is superior to intervening on $X_2$).  

\vspace*{-2mm}
\begin{figure}[h!] \centering
\includegraphics[width = 0.7\textwidth]{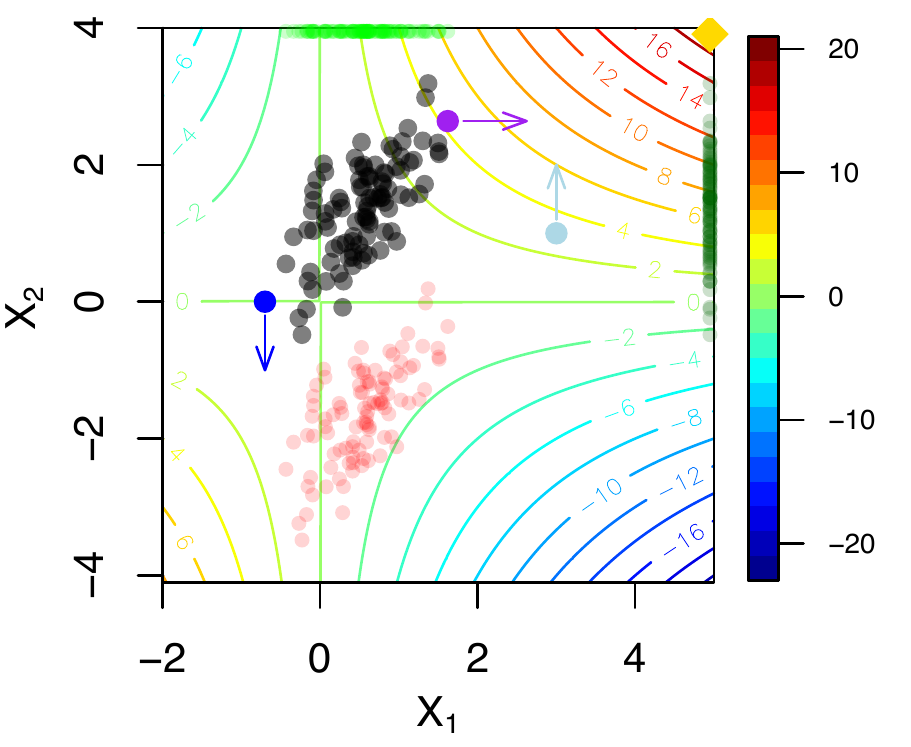} \vspace*{-3mm}
\caption{Contour plot of expected outcomes over feature space $ [X_1, X_2]$ for relationship  ${Y = X_1 \cdot X_2 +~\varepsilon}$.  Black points: the underlying population.  Gold diamond: optimal covariate-setting if any transformation in the box were feasible.  Red points:  same population after  shift intervention $\Delta = [-3,  0]$.  Light (or dark) green points (along border): best covariate-fixing intervention which can only set $X_2$ (or only $X_1$) to a fixed value. Blue, purple, light blue points: individuals who receive a single-variable personalized intervention (arrows indicate direction of optimal  transformation).}
\label{fig:examples}
\end{figure}

\section{Algorithms}
\label{sec:algorithms}

Throughout this work, we use Gaussian Process (GP) regression \citep{Rasmussen2006} to model $Y \mid X$ as described in \S\ref{sec:gp}  (`S'  indicates references in the Supplementary Material).
This nonparametric method has been favored in many applications as it produces both accurate predictions and effective measures of uncertainty (with closed-form estimators available in the standard case).  Furthermore, a variety of GP models exist for different  settings including: non-Gaussian response variables \citep{Rasmussen2006}, non-stationary relationships \citep{Paciorek2004}, deep representations \citep{Daminaou2013}, measurement error \citep{McHutchon2011}, and heteroscedastic noise \citep{Le2005}.  While these variants are not employed in this work, our methodology can be directly used in conjunction with such extensions (or more generally, any model which produces a useful posterior for $f \mid \mathcal{D}_n$).

Under the standard GP model, $G_x(T)$ follows a Gaussian  distribution and the $\alpha^{\text{th}}$ quantile of our personalized gain is simply given by: 
\begin{align*}
F^{-1}_{G_x(T)} =  & \mathbb{E}[ G_x(T) ] + \Phi^{-1}(\alpha) \cdot \text{Var} [ G_x(T)] \numberthis
\label{eq:quantilegp}
\end{align*}
 where $\Phi^{-1}$ denotes the $N(0,1)$ quantile function.  The quantiles of the empirical population gain may be similarly obtained.  When a smooth smooth covariance kernel $k(\cdot, \cdot)$ is adopted in the GP prior, derivatives of our intervention-objectives are easily computed with respect to $T$.

In many practical settings, an intervention that only affects a small subset of variables is desired.   Software to improve text, for example, should not overwhelm authors with a multitude of desired changes, but rather present  a concise list of the most beneficial revisions in order to retain underlying semantics.  
Note that identifying a sparse transformation of the covariates is different from feature selection in  supervised learning (where the goal is to identify dimensions along which $f$ varies most).  In contrast, we seek the dimensions $\mathcal{I} \subset \{1,\dots,d \}$ along which one of our feasible covariate-transformations can produce the largest high-probability increase in $f$, assuming the other covariates remain fixed at their initial pre-treatment values (in the case of personalized intervention) or follow the same distribution as the pre-intervention population (in the case of a global policy).

For a shift intervention $T(x) = x+ \Delta$, we introduce the convenient notation $G_n(\Delta) := G_n(T)$.  In applications  where shifting $x_s$ (the $s^{\text{th}}$ covariate for $s \in \{1,\dots, d\}$) by one unit incurs cost $\gamma_s$, we account for these costs by  considering the following regularized intervention-objective:
\begin{equation} \label{eq:shiftpopulation}
J_\lambda(\Delta) :=   F^{-1}_{G_n(\Delta)}(\alpha )  -  \lambda \sum_{s=1}^d \gamma_s | \Delta_s|
\end{equation}
By maximizing this objective over feasible set $\mathcal{C}_{\Delta} := \{ \Delta \in \mathbb{R}^d : x + \Delta \in \mathcal{C}_{x} \text{ for all } x \in \mathbb{R}^d \}$, policy-makers can decide which variables to intervene upon (and how much to shift them), depending on the relative value of outcome-improvements (specified by $\lambda$).

This optimization is performed using the proximal gradient method \citep{Bertsekas1995}, where at each iterate: a step in the gradient direction is followed by a soft-thresholding operation \citep{Bach2012} as well as a projection back onto the feasible set $\mathcal{C}_{\Delta}$.  However, a simple gradient method may suffer from local optima.  To avoid severely suboptimal solutions, we develop a continuation technique \citep{Mohabi2012}  that performs a series of gradient-based optimizations over variants of this objective with tapering levels of added smoothness (details in \S\ref{supsec:algorithms}).

In some settings, one may want to ensure at most $k < d$ covariates are intervened upon.  We identify the optimal $k$-sparse shift intervention via the Sparse Shift Algorithm below, which relies on $\ell_1$-relaxation \citep{Bach2012} and the regularization path of our penalized objective in (\ref{eq:shiftpopulation}).
  
\vspace*{4mm}

\begin{mdframed}[leftmargin=0cm,rightmargin=0cm,skipbelow=0cm,skipabove=0cm,
   innertopmargin=3pt, innerbottommargin=3pt, innerrightmargin=0.8pt, innerleftmargin=0.8pt]
\textbf{ Sparse Shift Algorithm:} Finds best $k$-sparse \\ \hspace*{0.5mm} shift intervention.   \\ \nopagebreak
\noindent\rule[1ex]{\linewidth}{1pt}
\vspace*{-6mm}
\begin{enumerate}[1:, topsep=-1.5ex,leftmargin=6mm]
\item Set $\gamma_s = 1$ for $s = 1,\dots, d$ 
\item  Perform binary search over $\lambda$ to find: 
\vspace*{-2mm}
\begin{align*}
\lambda^* \leftarrow \argmin \Big\{ \lambda \ge 0  & \text{ s.t.} \  \Delta^* := \argmax_{\Delta \in \mathcal{C}_{\Delta}} J_\lambda(\Delta) \\[-0.75em] 
 & \text{ has $\le k$ nonzero entries}  \Big\}
\end{align*}
\vspace*{-6mm}
\item Define $\displaystyle \mathcal{I} \leftarrow \text{support}(\Delta^*_{\lambda^*}) \subseteq \{1,\dots, d\}$ \\ where $\displaystyle \Delta^*_{\lambda^*} := \argmax_{\Delta \in \mathcal{C}_{\Delta}} J_{\lambda^*}(\Delta)$
\item \textbf{Return: } \ $\displaystyle \Delta^*  \in \mathbb{R}^d \leftarrow \argmax_{\Delta \in B} J_{\hspace*{0.1mm} 0}(\Delta)$  \\
where $B := \mathcal{C}_\Delta \bigcap \big\{ \Delta \in \mathbb{R}^d : \Delta_s = 0 \text{ if } s \notin \mathcal{I}  \big\}$
\end{enumerate}
\end{mdframed}

Recall that in the case of personalized intervention, we simply optimize over vectors $T(x) \in \mathcal{C}_x$.  Any personalized transformation can therefore be equivalently expressed as a shift in terms of $\Delta_x \in \mathbb{R}^d$ such that $T(x) = x + \Delta_x$.  After substituting the individual gain $G_x(\Delta_x)$ in place of  the population gain $G_n(\Delta)$ within our definition of $J_\lambda$ in (\ref{eq:shiftpopulation}), we can thus employ the same algorithms to identify sparse/cost-sensitive personalized interventions.  To find a covariate-fixing intervention which sets $k$ of the covariates to particular fixed constants across all individuals from the population, we instead employ a forward step-wise selection algorithm (detailed in \S\ref{supsec:uniformalgorithm}), as the form of the optimization is not amenable to $\ell_1$-relaxation in this case.

\section{Theoretical Results}
\label{sec:theory}
Consider the following basic conditions: 
 \refstepcounter{assumption} (A\theassumption) 
 \label{as:first}
 all data lies in $\mathcal{C} := [0,1]^d$, 
 \refstepcounter{assumption} (A\theassumption) ${0 < \alpha \le 0.5}$. 
\label{as:last}
Throughout this section, we assume (A\ref{as:first}), (A\ref{as:last}), and the conditions laid out in \S\ref{sec:intro}  hold.  For clarity, we rewrite the true underlying relationship as $f^*$, letting $f$ now denote arbitrary functions.  Our results are with respect to the \emph{true improvement} of an intervention ${G^*_x(T) := f^*(T(x)) - f^*(x)}$, $G^*_X(T) := \mathbb{E}_X [G^*_x(T)]$ (note that $G^*_x, G^*_X$ are no longer random).
Our theory relies on Gaussian Process results derived by \cite{Srinivas2010, VanderVaart2011}, and we relegate  proofs and technical definitions to \S\ref{subsec:proofs}.

\begin{thm} \label{thm:close} Suppose we adopt a GP$\big(0, k(x,x')\big)$ prior and the following conditions hold: \\ 
\refstepcounter{assumption} (A\theassumption)
 noise variables $\varepsilon^{(i)} \overset{iid}{\sim} N(0, \sigma^2)$ 
\refstepcounter{assumption} (A\theassumption)
there exist $\rho > 0$ such that the H\"{o}lder space $C^\rho[0,1]^d$ has probability one under our prior (see \cite{VanderVaart2011}).  
\refstepcounter{assumption} (A\theassumption)
$f^*$ and any $f$ supported by the prior are Lipschitz continuous over $\mathcal{C}$ with constant $L$
\refstepcounter{assumption} (A\theassumption)
the density of our input covariates $p_X \in [a,b]$ is bounded above and below over domain $\mathcal{C}$.

Then, for all $x, T(x) \in \mathcal{C}$: \ \ \ 
$$ \mathbb{E}_{\mathcal{D}_n} \Big| F^{-1}_{G_x(T)}(\alpha) - G^*_x(T) \Big|  \le \frac{C}{\alpha} \Big(L + \frac{1}{a} \Big) \cdot  \Psi_{\hspace*{-0.4mm}f^*}\hspace*{-0.2mm}(n)^{1 / [2(d+1)]}  
$$
\end{thm}
where constant $C$ depends on the prior and density $p_X$ and we define:
\[ \Psi_{f}(n) :=  \begin{cases} 
      \big[\psi_{f}^{-1}(n)\big]^2  & \hspace*{-8mm} \text{ if } \psi_{f}^{-1}(n) \le n^{-d/(4\rho + 2d)}  \\
      n \cdot [\psi_{f^*}^{-1}(n)]^{(4\rho + 4d)/d} & \text{ otherwise }  \\
   \end{cases} 
\]
$\psi_{f^*}^{-1}(n)$ is the (generalized) inverse of $\psi_{f^*}(\epsilon) := \frac{\phi_{f^*}(\epsilon)}{\epsilon^2}$ which depends on the concentration function
$\displaystyle \phi_{f^*}(\epsilon) = \inf_{h \in \mathcal{H}_k : || h - f^*||_\infty < \epsilon} ||h||^2_k - \log \Pi \big( f: ||f ||_\infty < \epsilon \big)$.  $\phi_{f^*}$ measures how well the RKHS of our GP prior $\mathcal{H}_k$ approximates $f^*$ (see \cite{VanderVaart2011} for more details).   The expectation $ \mathbb{E}_{\mathcal{D}_n}$  is over the distribution of the data $\{(X^{(i)}, Y^{(i)})\}_{i=1}^n$.   Importantly, Theorem \ref{thm:close} does not assume anything about the  true relationship $f^*$, and the bound depends on the distance between $f^*$ and our prior.   When $f^*$ is a $\rho$-smooth function, a typical bound is given by $\psi_{f^*}^{-1}(n) = \mathcal{O}( n^{- \min\{ \nu, \rho \}/(2\nu + d)})$ if $k$ is the Mat\'ern kernel with smoothness parameter $\nu$.  When $k$ is the squared exponential kernel and $f^*$ is $\beta$-regular (in Sobolev sense), $\psi_{f^*}^{-1}(n) = \mathcal{O}((1 / \log n)^{\beta/2 - d/4})$   \citep{VanderVaart2011}.  

\begin{thm} \label{thm:popclose}  Under the assumptions of Theorem \ref{thm:close}, for any $T$ such that $\Pr(T(X) \in \mathcal{C}) = 1$:
\begin{align*}
& \mathbb{E}_{\mathcal{D}_n} \Big| F^{-1}_{G_n(T)}(\alpha) - G_X^*(T)  \Big| \\
& \le \frac{C}{\alpha} \Big[L\sqrt{ \frac{d}{n} } +  \Big(L + \frac{1}{a} \Big)  \Psi_{f^*}(n)^{\frac{1}{2(d+1)}}   \Big]
\end{align*}
\end{thm}

Theorems \ref{thm:close} and \ref{thm:popclose} characterize the rate at which our personalized/population-intervention objectives are  expected to converge to the true improvement (due to contraction of the posterior as $n$ grows).  Since these results hold for all $T$, this implies the maximizer of our intervention-objectives will converge to the true optimal transformation as $n \rightarrow \infty$ (under a reasonable prior).  Complementing these results, Theorem \ref{thm:notbad} in \S\ref{subsec:proofs} ensures that for any $n$: optimizing our personalized intervention objective corresponds to improving a lower bound on the true improvement  with high probability, when $\alpha$ is small  and  $f^*$ belongs to the RKHS of our prior.  In this case, the optimal transformation inferred by our approach only worsens the actual expected outcome with low probability.

\section{Results}
\S\ref{supsec:simulation} contains an analysis of our approach on simulated data from simple covariate-outcome relationships.  The average improvement produced by our chosen interventions rapidly converges to the best possible value with increasing $n$.  In these experiments,  sparse-interventions consistently alter the correct feature subset, and proposed transformations under our conservative $\alpha = 0.05$ criterion are much  more rarely harmful than those suggested by optimizing the posterior mean function (which ignores uncertainty).

\subsection{Gene Perturbation}

Next, we applied our method to search for population interventions in observational yeast gene expression data  from \citet{Kemmeren2014large}. We evaluated the effects of proposed interventions (restricted to single gene knockouts) over a set $X$ of 10 transcription factors ($n=161$) with the goal of down-regulating each of a set of 16 small molecule metabolism target genes, $Y$.  Results for all methods are compared to the actual expression change of the target gene found experimentally under individual knockouts of each transcription factor in $X$. Compared to marginal linear regressions and multivariate linear regression, our method's uncertainty prevents it from proposing harmful interventions, and the interventions it proposes are optimal or near optimal (Figure~\ref{fig:geneExpression}).

Insets (a) and (b) in Figure~\ref{fig:geneExpression} show empirical marginal distributions between target gene \emph{TSL1} and members of $X$ identified for knockout by our method (\emph{CIN5}) and marginal regression (\emph{GAT2}).  From the linear perspective, these relationships are fairly indistinguishable, but only \emph{CIN5} displays a strong inhibitory effect in the knockout experiments.  Inset (c) shows the empirical marginal for a harmful intervention proposed by multivariate regression for down-regulating \emph{GPH1}, where the overall correlation
is significantly positive, but the few lowest expression values (which influence our GP intervention objective the most) do not provide strong evidence of a large knockdown effect.

\vspace*{-2mm}
\begin{figure}[h!] 
\hspace*{-2mm}
\includegraphics[trim=0in 7.775in 0in 0in, clip=true, width=1.02\textwidth]{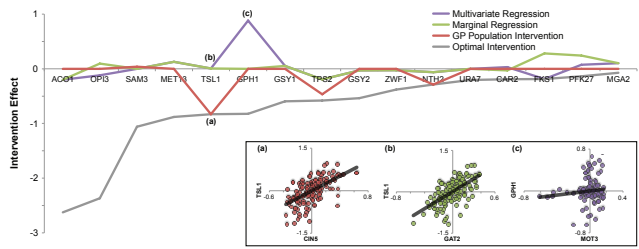}
\vspace*{-5mm} 
\caption{Actual effects of proposed interventions (single gene knockout) over a set transcription factors on down-regulation of each of a set of 16 small molecule metabolism target genes.}
\label{fig:geneExpression} 
\end{figure}

\subsection{Writing Improvement}
\label{sec:writing}

Finally, we apply our personalized intervention methodology to the task of transforming a given news article into one which will be more widely-shared on social media.  We use a dataset from  \cite{Fernandes2015} containing various features about individual Mashable articles along with their subsequent popularity in social networks (detailed description/results for this analysis in \S\ref{supsec:writing}). We train a GP regressor on 5,000 articles labeled with popularity-annotations and evaluate sparse interventions on a held-out set of 300 articles based on changes they induce in article \emph{benchmark popularity} (defined in  \S\ref{supsec:writing}).  When $\alpha = 0.05$, the average benchmark popularity increase produced by our personalized intervention methodology is 0.59, whereas it statistically significantly decreases to 0.55 if $\alpha = 0.5$ is chosen.  Thus, even given this large sample size, ignoring uncertainty appears detrimental for this application, and $\alpha = 0.5$ results in 4 articles whose benchmark popularity worsens post-intervention (compared to only 2 for $\alpha = 0.05$).  Nonetheless, both methods generally produce very beneficial improvements in this analysis, as seen in Figure \ref{fig:writing}.

As an example of the personalization of proposed interventions, our method ($\alpha = 0.05$) generally proposes different sparse interventions for articles in the Business category vs.\ the Entertainment category.  On average, the sparse transformation for business articles uniquely advocates decreasing global sentiment polarity and increasing word count (which are not commonly altered in the personalized interventions found for entertainment articles), whereas interventions to decrease title subjectivity are uniquely prevalent throughout the entertainment category. These findings appear intuitive (eg.\ critical business articles likely receive more discussion, and titles of popular entertainment articles often contain startling statements written non-subjectively as fact). Interestingly, the model also tends to advise shorter titles for business articles, but increasing the length for entertainment articles. Articles across all categories are universally encouraged to include more references to other articles and keywords that were historically popular.

\section{Misspecified Interventions}
\label{causalconnection}

Our methodology heavily relies on the assumption that the outcome-determining covariate values $\widetilde{x}$ produced through intervention exactly match the desired covariate transformation $T(x)$.  When transformations are only allowed to  alter at most $k < d$ covariates, this requires that we can intervene to alter only this subset without affecting the values of other covariates.  If $T$ specifies a sparse change affecting only a subset of the covariates $\mathcal{I} \subset \{1,\dots,d\}$, our methods assume the post-treatment value of any non-intervened-upon covariate remains at its initial value (ie.\ $\widetilde{x}_s = x_s \ \forall s \notin  \mathcal{I}$).

In some domains, the covariate-transformation induced via sparse external intervention can only be roughly controlled (eg.\ our gene perturbation example when the profiled genes belong to a common regulatory network).  Let $T_{\mathcal{I}\shortrightarrow z}$ denote a covariate-fixing transformation which sets a subset of covariates in $\mathcal{I} \subset \{1,\dots,d\}$ to constant values $z_\mathcal{I} \in \mathbb{R}^{|\mathcal{I}|}$ across all individuals in the population.  In this section, we consider an alternative assumption under which the intervention  applied in hopes of achieving $T_{\mathcal{I}\shortrightarrow z}$ propagates downstream to affect other covariates outside $\mathcal{I}$ (so there may exist $s \notin  \mathcal{I}$:  $\widetilde{x}_s \neq x_s$), which we formalize as the $\emph{do}$-operation in the causal calculus of \cite{Pearl2000}.  
Here, we suppose the underlying population of $X, Y$ follows a  \emph{structural equation model} (SEM) \citep{Pearl2000}.  
The outcome $Y$ is restricted to be a sink node of the causal DAG, so we can still write $Y = f^*(\widetilde{X}) + \varepsilon$ and maintain the other conditions from \S\ref{sec:intro}.  Rather than exhibiting covariate-distribution $T_{\mathcal{I}\shortrightarrow z}(X)$ with $Y =  f^*(T_{\mathcal{I}\shortrightarrow z}(X)) + \varepsilon$ (as presumed in our methods), the post-treatment population which arises from an intervention seeking to enact transformation $T_{\mathcal{I}\shortrightarrow z}$ is now assumed to follow the distribution specified by $p(X, Y \mid do(X_{\mathcal{I}} = z_{\mathcal{I}}))$.  Note that the $\emph{do}$-operation here is only applied to some nodes in the DAG (variables in subset $\mathcal{I}$) as discussed by \cite{Peters2014}, but its effects can alter the distributions of non-intervened-upon covariates outside of  $\mathcal{I}$ which lie downstream in the DAG.  

\begin{thm} \label{thm:dooperation}
 For some $\mathcal{I} \subseteq \{1, \dots,d\}$,  suppose the condition:  \refstepcounter{assumption} (A\theassumption) \label{as:bestdo} 
$pa(Y) \subseteq \mathcal{I} \ \bigcup \ \text{desc}(\mathcal{I})^C$ holds.
Then, for any covariate-fixing transformation $T_{\mathcal{I}\shortrightarrow z}$:
$\displaystyle \mathbb{E}_X \big[ f^*(T_{\mathcal{I}\shortrightarrow z}(x)) - f^*(x) \big]$ \ and \\  $\mathbb{E}_{\widetilde{x} \sim \text{do}(X_{\mathcal{I}} = z_{\mathcal{I}})} \big[ f^*(\widetilde{x}) \big] - \mathbb{E}_{X}\big[f^*(x) \big] 
$ are equal.
\end{thm}
Here, $\text{pa}(Y)$ denotes the variables which are parents of outcome $Y$ in the underlying causal DAG, and $\text{desc}(\mathcal{I})^C$ is the set of variables which are \emph{not} descendants of variables in subset $\mathcal{I}$.  For the next result, we define: 
$\displaystyle \mathcal{I^*} := \argmin \Big\{ |\mathcal{I}'| \text{ s.t. }  \exists \ T_{\mathcal{I}' \shortrightarrow z} \in  \argmax_{T_{\mathcal{I}\shortrightarrow z} :  |\mathcal{I}| \le k} \hspace*{-0.1mm} \mathbb{E}_X \big[ f^*(T_{\mathcal{I}\shortrightarrow z}(x)) - f^*(x) \big]  \Big\}$ 
as the intervention set corresponding to the optimal $k$-sparse covariate-fixing transformation (where in the case of ties, the set of smallest cardinality is chosen), if transformations were exactly realized by our interventions (which is not necessarily the case in this section).

\begin{thm} \label{thm:bestdofound} Suppose the  underlying DAG satisfies: 
\refstepcounter{assumption} (A\theassumption) \label{as:parentdisconnect} No variable in $\text{pa}(Y)$ is a descendant of other parents, ie. $\nexists \ j \in \text{pa}(Y)$ s.t.\ $j \in \text{desc}(\text{pa}(Y) \setminus \{j\})$.
Then, $\mathcal{I^*}$
satisfies (A\ref{as:bestdo}).
\end{thm}

In the absence of extremely strong interactions between variables in $\text{pa}(Y)$, the equality of Theorem \ref{thm:dooperation} will also hold for $\mathcal{I^*}$ if $| \text{pa}(Y) | \le k$.  For settings where sparse interventions elicit unintentional $do$-effects and the causal DAG meets  condition (A\ref{as:parentdisconnect}), Theorems \ref{thm:dooperation} and \ref{thm:bestdofound} imply that, under complete certainty about $f^*$, the (minimum cardinality) maximizer of our covariate-fixing intervention objective corresponds to an transformation that produces an equally good outcome change when the corresponding intervention is actually realized as a $do$-operation in the underlying population.  Combined with Theorem \ref{thm:popclose}, 
our results ensure that, even in this misspecified setting, the empirical maximizer of our sparse covariate-fixing intervention objective (\ref{eq:empiricalpolicy})  produces (in expectation as $n \rightarrow \infty$) beneficial interventions for populations whose underlying causal relationships satisfy certain conditions.

Next, we empirically investigate how effective  our methods are in this misspecified SEM setting, where a proposed sparse population transformation is actually realized as a \emph{do}-operation and can therefore  unintentionally affect other covariates in the post-intervention population.  We generate data from an underlying linear \emph{non}-Gaussian SEM, and where $Y$ is a sink node in the corresponding causal DAG  (see \S\ref{supsec:SEMdetails} for details).  Our approach to identify a beneficial sparse population intervention is compared with inferring the complete SEM using the LinGAM estimator of \cite{Shimizu2006} and subsequently identifying the optimal single-node \emph{do}-operation in the inferred SEM.  Note that LinGAM is explicitly designed for this setting, while both our method and the relied-upon Gaussian Process model are severely misspecified.  

Figures \ref{fig:lingam}A and \ref{fig:lingam}B demonstrate that the inferred best single-variable shift population intervention (under constraints on the magnitude of the shift) matches the performance the interventions suggested  by LinGAM (except for in rare cases with tiny sample size) when the proposed interventions are evaluated as \emph{do}-operations in the true underlying SEM.  Thus, we believe a supervised learning approach like ours is preferable in practical applications where interpreting the underlying causal structure is not as important as producing good outcomes (especially for higher  dimensional data where estimation of the causal structure becomes difficult \citep{Peters2014}).

\begin{figure}[h!] 
\centering
{\def\arraystretch{-0.3}
\begin{tabular}{c c} \hspace*{-3mm} 
 \includegraphics[width=0.49\textwidth]{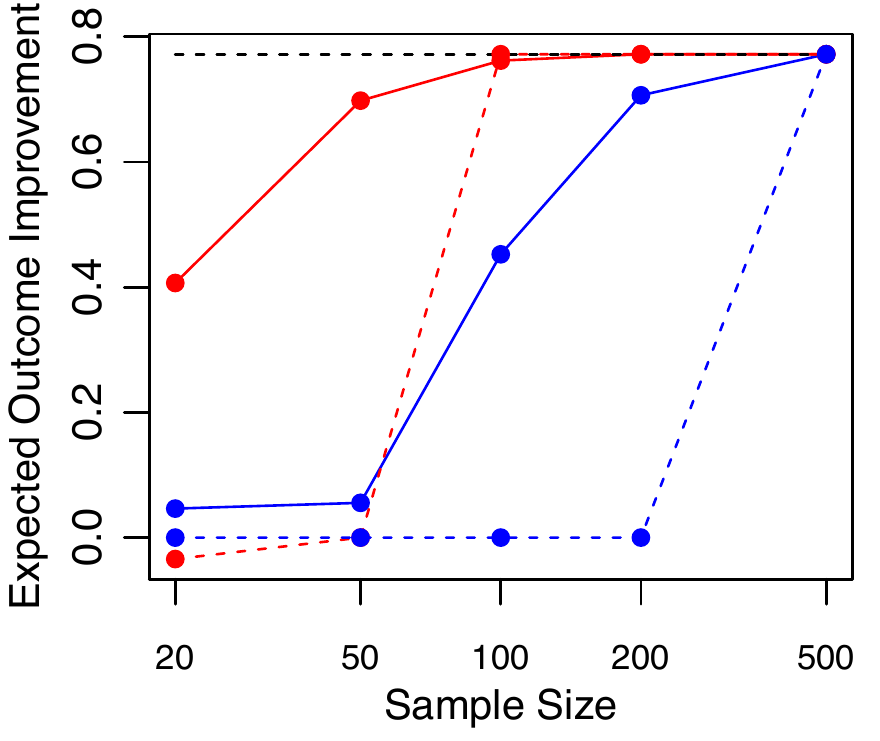} \hspace*{-3mm} 
& \hspace*{-1mm} \includegraphics[width=0.49\textwidth]{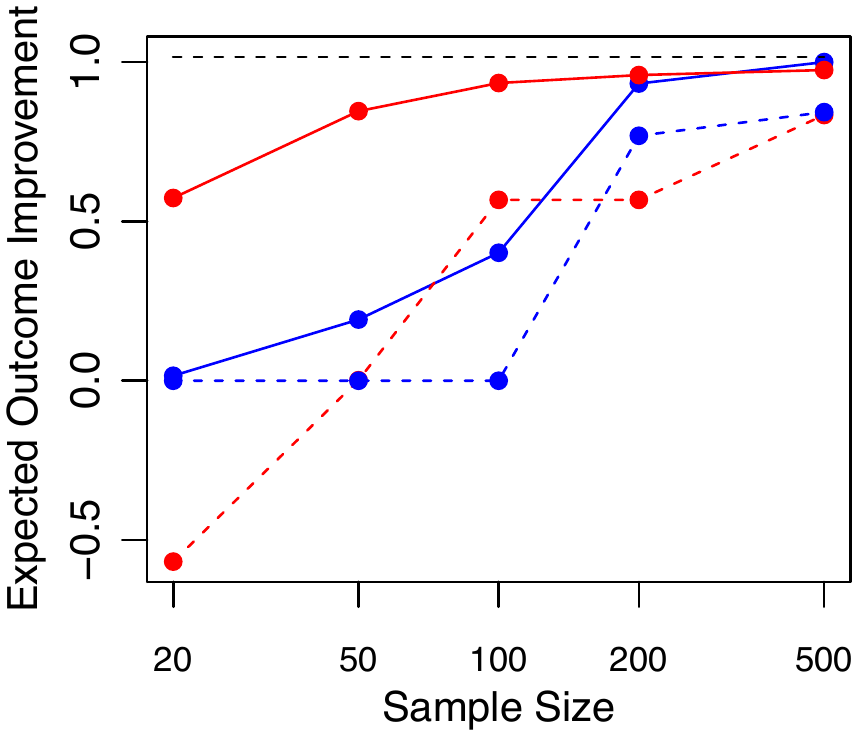} \\[0.25em] 
(A) SEM$_A$ \hspace*{0mm} & \hspace*{0mm} (B) SEM$_B$
\end{tabular} \vspace*{-2mm}}
\caption{The average (solid) and $0.05^{\text{th}}$ quantile (dashed) expected outcome change produced by our method (red) vs LinGAM (blue) over 100 datasets drawn from two underlying SEMs chosen by  \cite{Shimizu2006}.  The black dashed line indicates the best possible improvement in each case.}
\label{fig:lingam}
\end{figure}

The assumption of sparse interventions realized as a $do$-operation (as defined by \cite{Peters2014}) may also be an inappropriate in many domains, particularly if off-target effects of interventions are explicitly mitigated via external controls.  To appreciate the intricate nature of assumptions regarding  non-intervened-upon variables, consider our example of modeling text documents represented using two features: polarity and word count.   A desired transformation to increase the text's polarity can be accomplished by inserting additional positive adjectives, but such an intervention also increases articles' word count. Alternatively, polarity may be identically increased by replacing words with more positive alternatives, an external intervention which would not affect the word count (and thus follows the assumptions of our framework).


\section{Discussion}
\label{sec:discussion}
This work introduces methods for directly learning beneficial interventions from purely observational data without treatments.  While this objective is, strictly speaking, only possible under stringent assumptions, our approach performs well in both intentionally-misspecified and complex real-world settings.  As supervised learning algorithms grow ever more popular, we expect intervention-decisions in many domains will increasingly rely on predictive models.  Our conservative definition of the optimal intervention provides a principled approach to handle the inherent uncertainty in these settings due to finite data.  Able to employ any Bayesian regressor, our ideas are widely applicable, considering practical types of interventions that can either be personalized or enacted uniformly over a population.

\vspace*{3mm}

\textbf{Acknowledgements:  } We thank David Gifford for helpful comments. DR was funded by an award from IARPA (under research contract 2015-15061000003).

\clearpage \newpage 
\subsection*{References}
\bibliographystyle{agsm}
{\bibliography{InterventionOptBibliography}}


\clearpage \newpage \beginsupplement
\onecolumn
\begin{center}
{\huge \bf Supplementary Material}
\end{center}

\begingroup 
\let\orignumberline\numberline
\def\numberline#1{\orignumberline{#1}\kern1ex}
\setcounter{tocdepth}{0}
\tableofcontents
\addtocontents{toc}{\setcounter{tocdepth}{3}}
\endgroup

\section{Gaussian Process Regression}
\label{sec:gp}

Gaussian Process regression \citepsi{Rasmussen2006si} adopts a prior under which $f(x^{(1)}), \dots, f(x^{(n)})$ follow multivariate Gaussian distribution N$(\mathbf{m}_{n}, \mathbf{K}_{n,n})$ for any collection $\{x^{(i)}\}_{i=1}^n$.  The model is specified by a prior mean function $m : \mathbb{R}^d \rightarrow  \mathbb{R}$ and positive-definite covariance function $k :  \mathbb{R}^d \times  \mathbb{R}^d  \rightarrow  \mathbb{R}$ which encodes our prior belief regarding properties of the underlying relationship between $X$ and $Y$ (such as smoothness or periodicity).  Here, the  vector  $\mathbf{m}_{n} \in \mathbb{R}^n$ denotes the evaluation of function $m$ at each point $\{x^{(i)}\}_{i=1}^n$, and $\mathbf{K}_{n,n}$ denotes the matrix whose $i,j^{\text{th}}$ component is $k(x^{(i)}, x^{(j)})$.  Given test input points $x_*^{(1)},\dots,x_*^{(n_*)} \in \mathbb{R}^d$ in addition to training data $\mathcal{D}_n$, we additionally define:  $\mathbf{f}_* := [f(x_*^{(1)}), \dots, f(x_*^{(n_*)})]$, $\mathbf{y}_n = [y^{(1)}, \dots, y^{(n)}]$, matrix $\mathbf{K}_{n,*}$ with  $i,j^{\text{th}}$ entry $k(x^{(i)}, x_*^{(j)})$ (where $x^{(i)}$ is the $i^{\text{th}}$ training input), and matrix $\mathbf{K}_{*,*}$ which contains pairwise covariances between test inputs.

Assuming the noise $\varepsilon \sim \text{N}(0, \sigma^2$) is independently sampled for each observation, the posterior for $f$ at the test inputs, $\mathbf{f}_* \mid \mathcal{D}_n$, follows \ $\text{N}(\mathbf{\mu_n}_*, \mathbf{\Sigma_n}_* )$ distribution with the following mean vector and covariance matrix: 
\begin{align*}
\mathbf{\mu_n}_* = \mathbf{m}_{*} + (\mathbf{K}_{n,n} + \sigma^2 \mathbf{I})^{-1} (\mathbf{y}_n - \mathbf{m}_{n}), \ \mathbf{\Sigma_n}_* = \mathbf{K}_{*,*} - \mathbf{K}_{*,n} (\mathbf{K}_{n,n} + \sigma^2 \mathbf{I})^{-1} \mathbf{K}_{n,*}
\end{align*}

Note that our intervention-optimization framework is not specific to this GP model, but can be combined with any algorithm that learns a reasonable posterior for $f$.  While employing a more powerful model should improve the results of our approach, comparing various regressors is not our focus.  Thus, all practical results of our methodology are presented using only the standard GP regression model, under which the posterior distribution over $f$ is given by the above expressions.  In each application presented here, our GP uses the Automatic-Relevance-Determination (ARD) covariance function, a popular choice for multi-dimensional data \citepsi{Rasmussen2006si}:
\begin{equation}
k(x, x') = \sigma_0^2 \cdot \exp \left[-\frac{1}{2} \sum_{s=1}^d \left( \frac{x_s - x_s'}{l_s} \right)^2 \right]
\end{equation}
The ARD kernel relies on length-scale hyperparameters $l_1, \dots, l_d$ which determine how much  $f$ can depend on each dimension of the feature-space.  All hyperparameters of our GP regression model (covariance-kernel parameters $l_1 \dots, l_d$ and $\sigma_0$ (the output variance) as well as the variance of the noise $\sigma^2$) are empirically selected via marginal-likelihood maximization \citepsi{Rasmussen2006si}. In each application, we employ the $0.05^\text{th}$ posterior-quantile ($\alpha = 0.05$) in our method to ensure that with high probability, the intervention it infers to be optimal induces a nonnegative change in expected outcomes.

\section{Algorithmic Details}
\label{supsec:algorithms}

To find an optimal transformation of our regularized objective $J_\lambda$ in (\ref{eq:shiftpopulation}), we employ the proximal gradient method described in \S\ref{sec:algorithms}.  When $\lambda = 0$ and there is no penalty, we instead use Sequential Least Squares Programming \citepsi{Kraft1988si}.  However, the intervention objective $J_\lambda$ may be highly nonconcave.  To deal with  local optima in acquisition functions, Bayesian optimization methods employ heuristics like combining the results of many local optimizers or operating over a fine partitioning of the feature space \citepsi{Shahriari2016si, Lizotte2008}.  We instead propose a continuation technique that solves a series of optimization problems, each of which operates on our objective under a smoothed posterior (and the amount of additional smoothing is gradually decreased to zero).  Excessive  smoothing of the posterior is achieved by simply considering GP models whose kernels are given overly large length-scale parameters.  Each time the amount of smoothing is tapered, we initialize our local optimizer using the solution found at the previously greater smoothing level.  Intuitively, the highly smoothed GP model is primarily influenced by the global structure in the data, and thus our optimization with respect to the posterior of this model is far less susceptible to low-quality local maxima.  
Analysis of a similar homotopy strategy under radial basis kernels has been conducted by \citesi{Mohabi2012si}.

\subsection{Sparse Shift Intervention}
\label{supsec:shiftalgorithm}

Here, we provide an explanatory description of the Sparse Shift Algorithm from \S\ref{sec:algorithms}.
To find the best $k$-sparse population shift intervention, we resort to $\ell_1$ relaxation.  As the $\ell_1$-norm provides the closest convex relaxation to the $\ell_0$ norm, this is a a commonly adopted strategy to avoid combinatorial search in feature selection \citepsi{Bach2012si}.  First, we compute the regularization path over different settings of the penalty $\lambda > 0$ for the following regularized objective:

\begin{equation} \label{eq:shiftpopulation}
J_\lambda(\Delta) :=   F^{-1}_{G_n(\Delta)}(\alpha )  - \lambda ||\Delta||_1
\end{equation}
which is maximized over the feasible set $\mathcal{C}_{\Delta} := \{ \Delta \in \mathbb{R}^d : x + \Delta \in \mathcal{C}_{x} \text{ for all } x \in \mathbb{R}^d \}$ \\
(recall we write $G_n(\Delta) := G_n(T)$ when $T(x) = x+ \Delta$).

Subsequently, we identify the regularization penalty which produces a shift of desired cardinality and
select our intervention set $\mathcal{I}$ as the covariates which receive nonzero shift.  Finally, we optimize the original unregularized objective ($\lambda = 0$) with respect to only the selected covariates in $\mathcal{I}$ to remove bias induced by the regularizer.  Each inner maximization in both the Sparse Shift/Covariate-fixing algorithms is performed via the proximal gradient methods combined with our continuation approach introduced in \S\ref{supsec:algorithms}.

\subsection{Sparse Covariate-fixing Intervention}
\label{supsec:uniformalgorithm}

Another goal is to identify the optimal covariate-fixing intervention which sets $k$ of the covariates to particular fixed constants uniformly across all individuals from the population.  We employ the    forward step-wise selection algorithm described below, as the form of the optimization in this case is not amenable to $\ell_1$-relaxation.  Recall $\mathcal{I} \subseteq \{1,\dots,d\}$ denotes the subset of covariates which are intervened upon, and the covariate-fixing intervention produces vector $T_{\mathcal{I} \shortrightarrow z}(x) \in \mathbb{R}^d$ such that $T_{\mathcal{I} \shortrightarrow z}(x)_s = x_s$ if $s \notin \mathcal{I}$, otherwise  $T_{\mathcal{I} \shortrightarrow z}(x)_s = z_s$ which is a constant chosen by the policy-maker.  This same transformation is applied to each individual in the population, creating a more homogeneous group which share the same value for the covariates in $\mathcal{I}$.  For a given $\mathcal{I}$, the objective function to find the best constants is:
\begin{align*}
& J^{\text{unif}}_\mathcal{I}\big(\{z_s\}_{s \in \mathcal{I}} \big) :=  F^{-1}_{G_n(T_{\mathcal{I}\shortrightarrow z})} (\alpha) \numberthis \label{eq:uniformpopulation} \\
 \text{ with } \ & G_n(T_{\mathcal{I}\shortrightarrow z}) = \frac{1}{n} \sum_{i=1}^n \big[ f(z^{(i)}) - f(x^{(i)}) \big] \mid \mathcal{D}_n \  \text{ where } z_s^{(i)} =  \begin{cases} 
      x^{(i)} & \text{ if } s \notin \mathcal{I} \\
      z_s & \text{ otherwise }  \\
   \end{cases}
\end{align*}
which is maximized over the constraints: $z_s \in \mathcal{C}_s \subseteq \mathbb{R}$ for $s \in \mathcal{I}$.

\noindent\rule[0.5ex]{\linewidth}{1pt} \nopagebreak
\textbf{Sparse Covariate-fixing Algorithm:} Identifies best $k$-sparse covariate-fixing intervention.  \\ \nopagebreak
\noindent\rule[0.5ex]{\linewidth}{1pt}
\textbf{Input:} Dataset $\mathcal{D}_n = \{ (x^{(i)}, y^{(i)}) \}_{i=1}^n$, Posterior $f \mid \mathcal{D}_n$ \\
\textbf{Parameters:} $k \in \{1, \dots, d\}$ specifies the maximal number of covariates which may be set by the covariate-fixing intervention, $\mathcal{C}_1,\dots, \mathcal{C}_d \subseteq \mathbb{R}$ are sets of feasible settings for each covariate.
\begin{enumerate}[1:, topsep=-1.5ex,leftmargin=6mm]
\item  Initialize $\mathcal{I} \leftarrow \varnothing$, \ $\mathcal{U} \leftarrow \{1,\dots, d\}$, \ $J^* \leftarrow 0$
\item \textbf{While} $|\mathcal{I}| < k$:
\item \hspace*{5mm} Set $\displaystyle J^*_s \leftarrow \max_{\mathcal{C}_r : r \in \mathcal{I} \cup \{s\}} J^{\text{unif}}_{\mathcal{I} \cup \{s\}} \big(\{z_r\}_{r \in \mathcal{I} \cup \{s\}} \big)$ \hspace*{5mm} \textbf{for} each $s \in \mathcal{U}$
\item \hspace*{5mm} Find $ s^* \leftarrow \argmax_{s \in \mathcal{U}} \big\{ J^*_s  \big\} $ 
\item \hspace*{5mm} \textbf{If } $J^*_{s^*} > J^*$: \hspace*{5mm} $J^* \leftarrow J^*_{s^*}$, \ $\mathcal{I} \leftarrow \mathcal{I} \cup \{s^*\}$, \ $\mathcal{U} \leftarrow \mathcal{U} \setminus s^*$
\item \hspace*{5mm} \textbf{Else: }  \hspace*{4mm} \textbf{ break } \nopagebreak
\item \textbf{Return: } $\{z_s^*\}_{s \in \mathcal{I}} \leftarrow  \argmax_{\mathcal{C}_s : s \in \mathcal{I}} \  J^{\text{unif}}_{\mathcal{I}} \big(\{z_s\}_{s \in \mathcal{I}} \big)$
\end{enumerate}
\noindent\rule[0.5ex]{\linewidth}{1pt}

\section{Simulations}
\label{supsec:simulation}

\begin{figure}[h!]   
\centering
\begin{tabularx}{\textwidth}{X X} 
\includegraphics[width=0.4\textwidth]{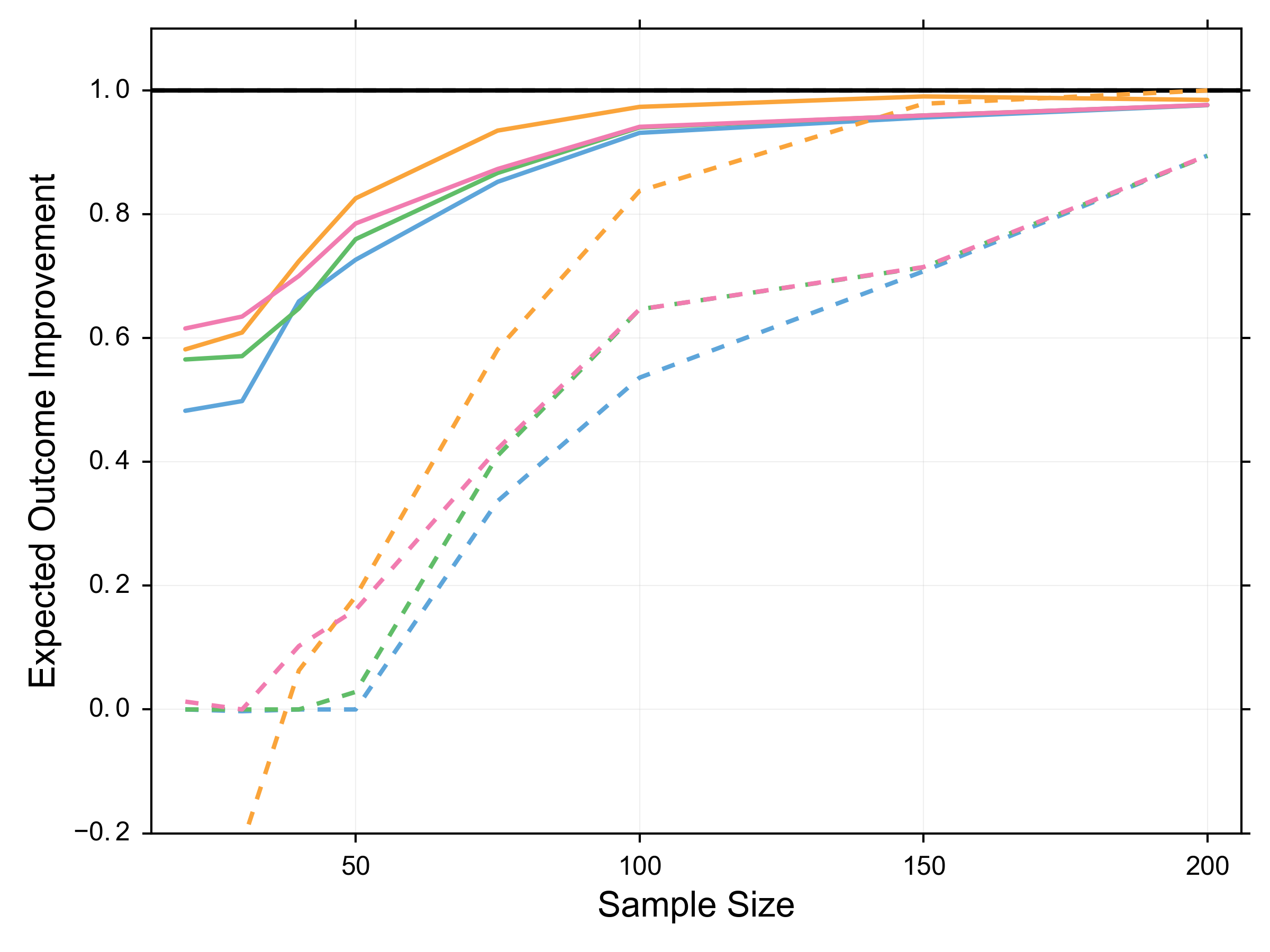} & 
\includegraphics[width=0.4\textwidth]{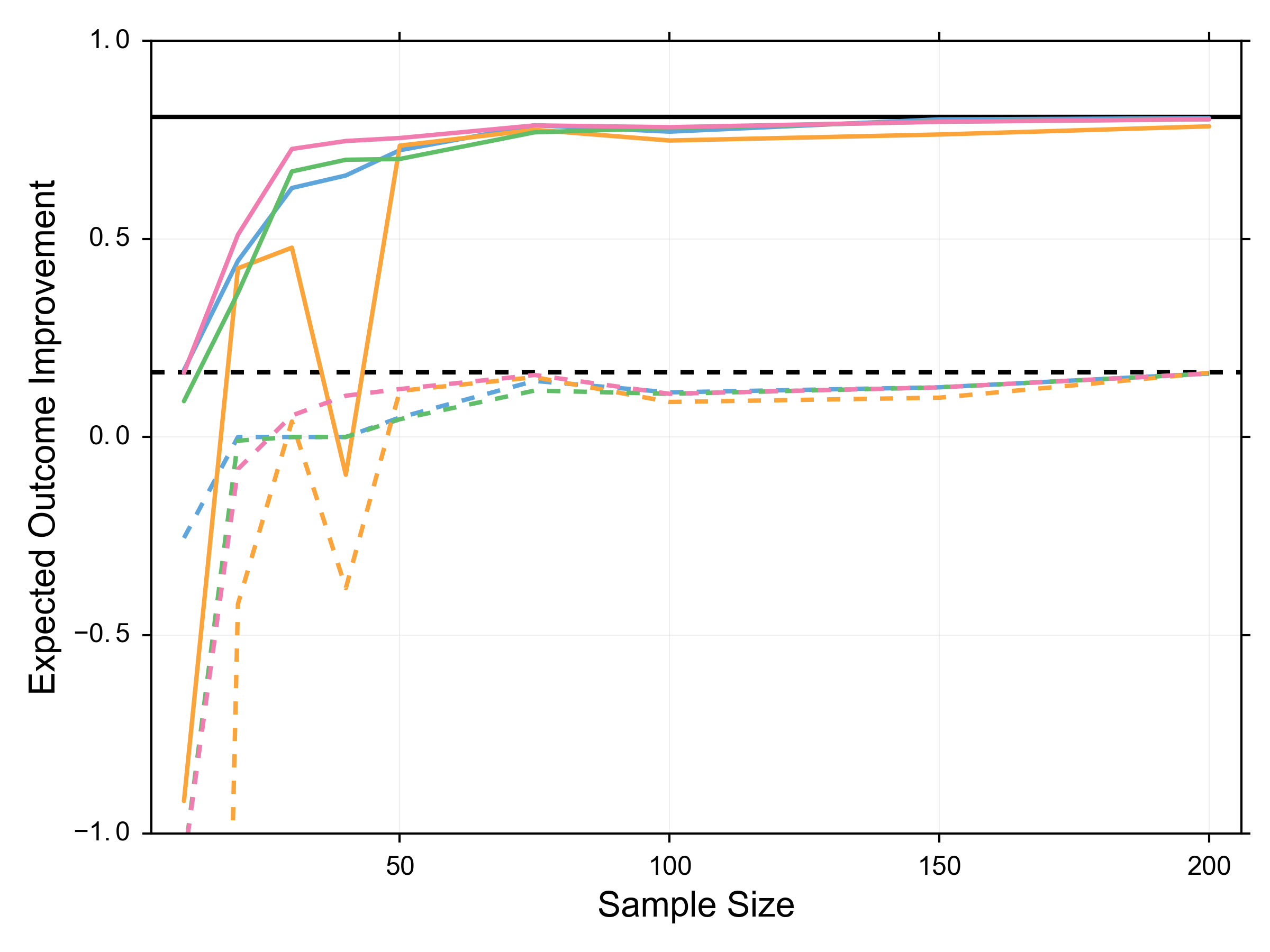} \\ 
(A) Linear: \ $f(X) = 0.3X_1 + 0.7X_2$ &  (B) Quadratic: \  $f(X) =1-X_1^2-X_2^2$ \\
\includegraphics[width=0.4\textwidth]  
{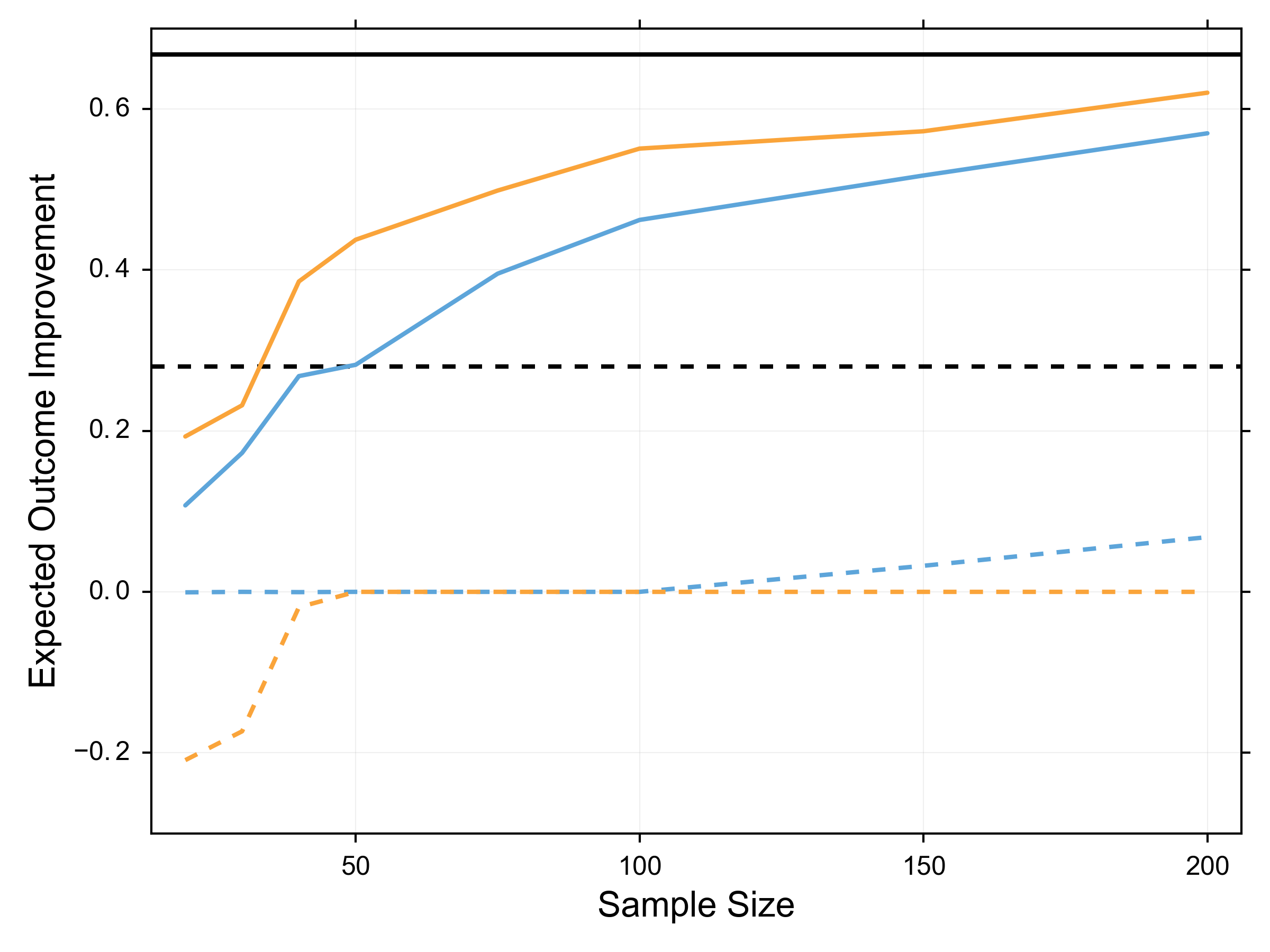} &  \ \ \ \ \ \ \ \ \ \
\includegraphics[width=0.3\textwidth]{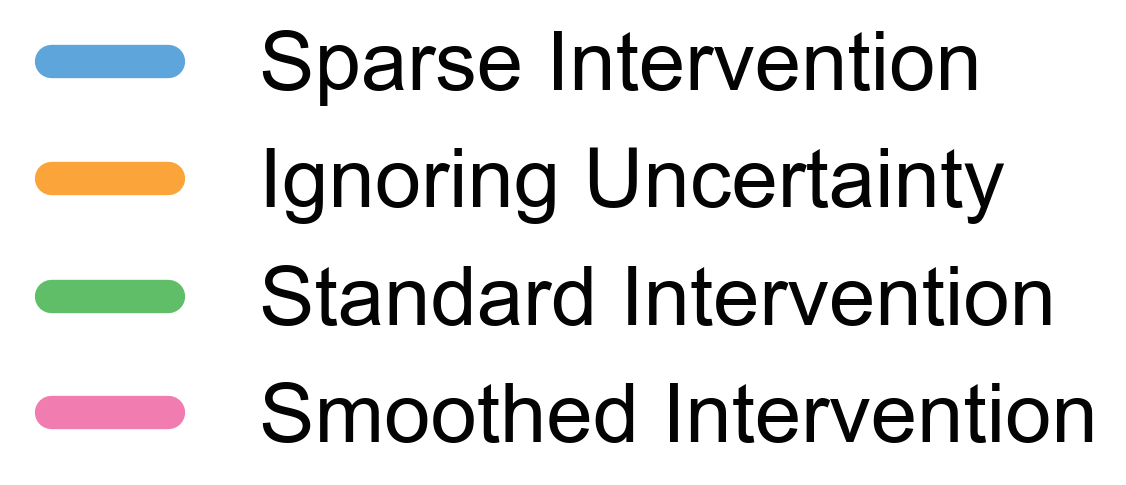}  \\
(C) Product: \ $f(X) =X_1 \cdot X_2$ & 
\end{tabularx}
\vspace*{-0.5mm}
\caption{The mean (solid) and $0.05^{\text{th}}$ quantile (dashed) expected outcome change produced under personalized interventions suggested by various methods, over 100 datasets of each sample size.  Each dataset contains 10-dimensional covariates, with $X_i \sim \text{Unif}[-1, 1]$, and $Y$ is determined by the indicated relationships and additive Gaussian noise ($\sigma = 0.2$).  The black lines indicate the best possible expected outcome change (when the best change depends on which individual received the intervention, the black solid/dashed lines indicates the mean and $0.05^{\text{th}}$ quantile over our 100 trials).}
\label{fig:individualIntervention}
\end{figure}


\begin{figure}[h!] 
\centering
\begin{tabular}{c c} 
\includegraphics[width=0.4\textwidth]{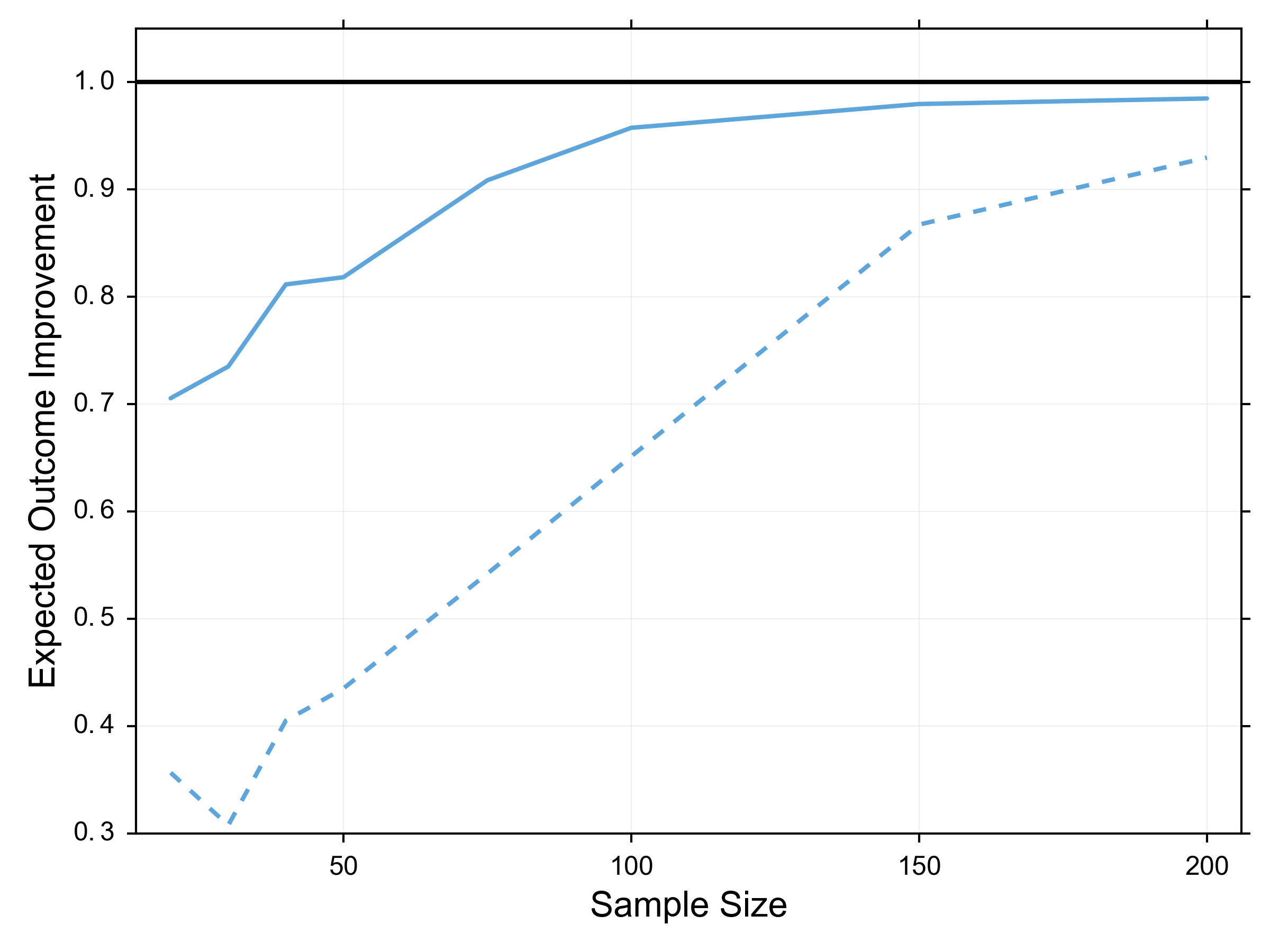}
& \includegraphics[width=0.4\textwidth]{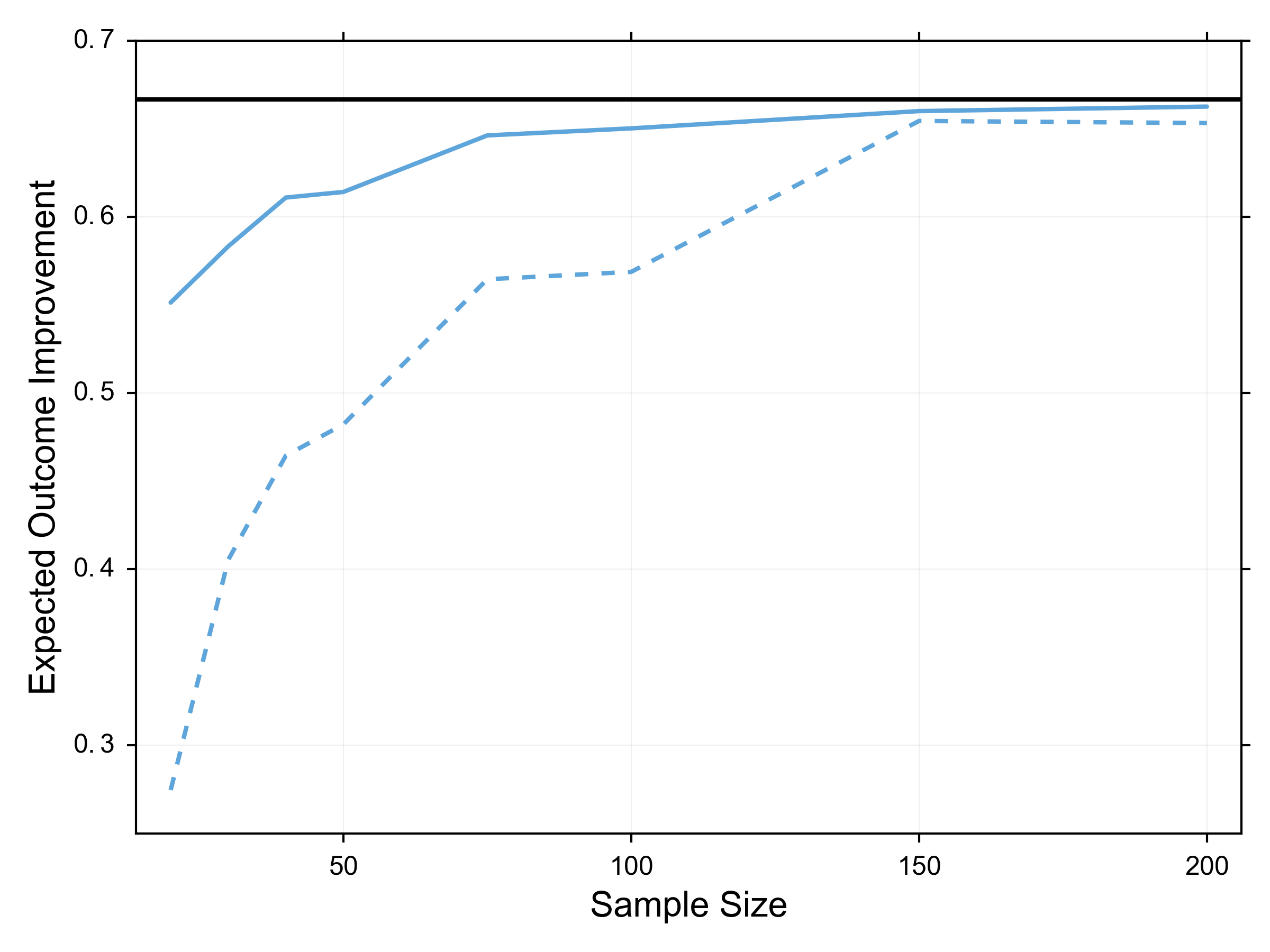} \\ 
(A) Linear: \ $f(X)=0.3X_1 + 0.7X_2$ & (B) Quadratic:  \ $Y=1-X_1^2-X_2^2$ \\
\end{tabular}
\vspace*{-0.5mm}
\caption{The mean (solid) and $0.05^{\text{th}}$ quantile (dashed) expected outcome change produced by our population intervention method, over 100 datasets for each sample size (same setting as in Figure \S\ref{fig:individualIntervention}).  The black line indicates the best possible expected outcome improvement.}
\label{fig:populationIntervention} 
\end{figure}

Over the simulated data summarized in Figure \ref{fig:individualIntervention}, we apply our basic personalized intervention method ($\alpha = 0.05$) with purely local optimization (standard) and our continuation technique (smoothed), which significantly improves results.  For each of the 100 datasets, we randomly sampled a new point (from the same underlying distribution) to receive a personalized intervention.   
The magnitude of each intervention is bounded by 1, except for in data from the quadratic relationship.  We also infer sparse interventions (with a cardinality constraint of 2 for the linear and quadratic relationships, 1 for the product relationship).  When $Y = X_1 \cdot X_2 + \varepsilon$, the optimal (constrained) intervention may drastically vary depending upon the individual's covariate-values, and our algorithm is able to correctly infer this behavior (Simulation C).  Finally, we also apply a variant of our method which entirely ignores uncertainty ($\alpha = 0.5$).  While this approach is on average  better for larger sample sizes, highly harmful interventions are occasionally proposed, whereas our uncertainty-adverse method ($\alpha =0.05$) is much less prone to producing  damaging  interventions (preferring to abstain by returning $T(x) = x$ instead).  This is an invaluable characteristic since interventions generally require effort and are only worth conducting when they are likely to produce a substantial benefit.

Figure \ref{fig:populationIntervention} displays the behavior of both the population shift intervention in the linear setting, and the population covariate-fixing intervention under the quadratic  relationship.  The population intervention is notably safer than the individually tailored variants, producing no negative changes in our experiments.

\vspace*{-3mm}
\clearpage
\subsection{Linear SEM Analysis} \label{supsec:SEMdetails}
\addcontentsline{toc}{subsection}{\nameref{supsec:SEMdetails}}
\vspace*{-1mm}

Here, we suppose that a desired transformation upon variable $s \in \{1,\dots, d\}$ cannot be enacted exactly and the $Y$ which arises post-treatment is distributed according to $do(X_s = \mathbb{E}[X_s] + \Delta)$, where $\mathbb{E}[X_s]$ is the mean of the pre-treatment marginal distribution of the $s$th covariate.  In this case,  \emph{do}-effects can propagate to other covariates which are descendants of $s$ in the DAG because the values of descendant variables are redrawn from the \emph{do}-distribution which arises as a result of shifting $\mathbb{E}[X_s]$.  Because all relationships are linear in our SEMs, the actual expected outcome change resulting from a particular shift (resulting from the corresponding \emph{do}-operation) is easily obtained in closed form.

Our GP framework is applied to the data to infer an optimal 1-sparse shift population intervention (only interventions on a single variable are allowed).  The maximal allowed magnitude of the shift is constrained to ensure the optimum is well-defined (to $\pm 1$ times the standard deviation of each variable in the underlying SEM distribution).
An alternative approach to improve outcomes in contrast to our black-box approach is to apply a causal inference method like LinGAM \citepsi{Shimizu2006si} to estimate the SEM from the data, and then identify the optimal single-variable shift $\Delta_s^*$ in the LinGAM-inferred SEM (since all inferred relationships are also linear, the optimal single-variable shift will be either 0 or the lower/upper allowed shift and we simply search over these possibilities). We compare our approach against LinGAM by evaluating the actual expected outcome change produced by the shift $\Delta^*_s$  proposed by each method (where the actual expected outcome change is found by analytically  performing the $do(X_s = x_s + \Delta^*_s)$ operation in the true underlying SEM) .  

In our experiment, two underlying SEM models  are considered which were used by \citesi{Shimizu2006si} to demonstrate the utility of their LinGAM method (albeit with impractically large sample size = 10,000).  SEM$_A$ is used to refer to the model depicted in Figure 3 of \citesi{Shimizu2006si}, where we define $Y$ as x6 (a sink node in the causal DAG).  SEM$_B$ denotes the underlying model of Figure 4 in the same paper ($Y$ is defined as sink node x7).  The remainder of the variables in each SEM are adopted as our observed covariates $X$.
 
\nopagebreak
This experiment represents an application of our method in a highly misspecified setting.  The true data-generating mechanism differs significantly from assumptions of our GP regressor (output noise is now fairly non-Gaussian, the underlying relationships are all linear while we use an ARD kernel).  Furthermore, an intervention to transform a single covariate incurs a multitude of unintentional off-target effects resulting from the \emph{do}-effects propagating to downstream covariates in the SEM, whereas our method believes only the chosen covariate is changed.  In contrast, this data exactly follows the special assumptions required by LinGAM, and we properly account for inferred downstream \emph{do}-operation effects when identifying the best inferred intervention under LinGAM.  The only disadvantage of the LinGAM method is that it does not know the direction of the causal relationship $X \rightarrow Y$ (although we found it always estimated this direction correctly except on rare occasions with tiny sample sizes of $n = 20$).  

Since LinGAM only estimates linear relations, the best inferred shift-intervention found by this  approach will always be 0 or the minimal/maximal shift allowed for a particular covariate. Searching over these three values for each covariate ensures the actual optimal shift will be recovered if the LinGAM SEM-estimate were correct. However, under our approach, identifying the optimal population shift-intervention requires solving an optimization problem.  Even if the GP regression posterior were to exactly reflect the true data-generating mechanism, our approach might get stuck in a suboptimal local maximum or avoid
\nopagebreak
the minimal/maximal allowed shift  due to too much uncertainty about $f$ in the resulting region of feature-space.  In practice, these potential difficulties do not pose much of an issue for our approach.

\section{Gene Knockout Interventions}
\label{supsec:geneExpression}

The data set used for this analysis contains gene expression levels for a set of wild type (ie.\ `observational') samples, $\mathcal{D}_{obs} \ (n=161)$, as well as for a set of `interventional' samples, $\mathcal{D}_{int}$, in which each individual gene was serially knocked out. In our analysis, we search for potential interventions for affecting the expression of a desired target gene by training our GP regressor on $\mathcal{D}_{obs}$ and determining which knockout produces the best value of our empirical covariate-fixing population intervention objective (for down-regulating the target).
Subsequently, we use $\mathcal{D}_{int}$ to evaluate the actual effectiveness of proposed interventions in the knockout experiments. We only search for interventions present in $\mathcal{D}_{int}$ (single gene knockouts) rather than optimizing to infer optimal covariate transformations.   

As candidate genes for this analysis we used only the 700 genes that \citesi{Kemmeren2014largesi} classified as responsive mutants (at least four transcripts show robust changes in response to the  knockout).  Furthermore, we omitted genes whose expression over the 161 observational samples had standard deviation $< 0.1$.  Out of the transcription factors present in the remaining set of genes, we defined the top 10 factors as our feature set $X$, after ranking the transcription factors by the difference between their expression when they were knocked out in the interventional data and their $0.1^{\text{th}}$ quantile expression level in the observational data.  This was to ensure that our  model would be trained on data that at least resembled the experimental data $\mathcal{D}_{int}$.  The set of genes to down-regulate was simply chosen to be those classified by \citesi{Kemmeren2014largesi} as small molecule metabolism genes that met the minimum standard deviation requirement in their observational expression marginal distribution.  The resulting set was 16 target genes, and the (negative) expression of each of was treated as an outcome $Y$ in our analyses.

Each method evaluated in this analysis was to propose an intervention (single gene knockout) to down-regulate the expression of each target gene (separately).  Once a gene to knock out was proposed, this  intervention was evaluated by comparing the resulting expression of the target when the proposed knockout was actually performed in the experimental data  $\mathcal{D}_{int}$.  This expression level could then be compared to the `optimal' choice of gene from $X$ to intervene upon (the gene in $X$ whose  knockout produced the largest down-regulation of the target in $\mathcal{D}_{int}$).  

We compared our approach against two methods popularly used to draw conclusions about affecting  outcomes in the sciences.  First, we applied a multivariate regression analysis in which a linear regression model was fit to the observations of $(X, Y)$ in $\mathcal{D}_{obs}$.  The best gene to knockout was inferred on the basis of the regression coefficients and expression values (if no beneficial regression coefficient was found significant at the 0.05 level under the standard $t$-test, then no intervention was proposed).  Second, we performed a marginal analysis in which separate univariate linear regression models were fit to $(X_1, Y), \dots, (X_d, Y)$, and the best knockout was again inferred on the basis of the regression coefficients and expression values (again, no intervention was recommended if there was no statistically significant beneficial regression coefficient at the 0.05 level, after correcting for multiple testing via the False Discovery Rate).

Figure~\ref{fig:geneExpression} compares the results produced by these methods to the optimal intervention over $X$ for down-regulating each $Y$, as found in the experimental data $\mathcal{D}_{int}$.  Of the 16 small molecule metabolism target genes tested, in three cases our method proposed an intervention which was found to be optimal or near optimal in $\mathcal{D}_{int}$, while in the remaining cases, the model uncertainty causes the method not to recommend any intervention (except for one very minorly harmful intervention for target \emph{SAM3}).   On the other hand, neither form of linear regression proposed effective interventions for any target other than \emph{FKS1}, and in some cases, the linear regressors proposed counterproductive interventions that up-regulated the target. This highlights the importance of a model that properly accounts  uncertainty when evaluating potential interventions.

\section{Interventions to Improve Article Popularity}
\label{supsec:writing}

We demonstrate our personalized intervention methodology in a setting with rich nonlinear underlying relationships.  The data consist of 39,000 news articles published by Mashable around 2013-15 \citepsi{Fernandes2015si}. Each article is annotated with the number of shares it received in social networks (which we use as our outcome variable after log-transform and rescaling).  A multitude of features have been extracted from each article (eg.\ word count, the category such as ``tech'' or ``lifestyle'', keyword properties), many of which \citesi{Fernandes2015si} produced using natural language processing algorithms (eg.\ subjectivity, polarity, alignment with topics found by Latent Dirichlet Allocation).  After removing many highly redundant covariates, we center and rescale all variables to unit-variance (see Table \ref{tab:textfeatures} for a complete description of the 29 covariates used in this analysis).  
 
We randomly partition the articles into 3 disjoint groups: a \emph{training} set (5,000 articles on which scaling-factors are computed and our GP regressor is trained), an \emph{improvement} set (300 articles we find interventions for), and a \emph{held-out} set (over 34,000 articles used for evaluation).  A large group is left out for validation to ensure there are many near-neighbors for any given article, so we can reasonably estimate the true expected popularity given any setting of the article-covariates.  Subsequently, a basic GP regression model is fitted to the training set.  As the predictive power of our GP regressor did not measurably benefit from ARD feature-weighting, we simply use the  squared exponential kernel.  Over the held-out articles, the Pearson correlation between the observed popularity and the GP (posterior mean) predictions is 0.35.  Furthermore, there is a highly significant ($p < 8 \cdot 10^{-41}$) positive correlation of 0.07 between the model's predictive variance and the actual squared errors of GP predictions over this held-out set.  Our model is thus able to make reasonable predictions of popularity based on the available covariates, and its  uncertainty estimates tend to be larger in areas of the feature-space where the posterior mean lies further from actual popularity values.

In this analysis, we compare our personalized intervention methodology which \emph{rejects} uncertainty (using $\alpha = 0.05$) with a variant of the this approach that \emph{ignores} uncertainty (using the same objective function with $\alpha = 0.5$).  Both methods share the same GP regressor, optimization procedure, and set of constraints.  For the 300 articles in the intervention set (not part of the training data) we allow intervening  upon all covariates except for the article category which presumably is fixed from an  author's perspective.  All covariate-transformations are constrained to lie within [-2,2] of the original (rescaled) covariate value, and we impose a sparsity constraint that at most 10 covariates can be intervened upon for a given article. 

Unfortunately, no pre-and-post-intervention articles are available for us to ascertain a ground truth evaluation.  To crudely measure performance, we estimate the underlying expected popularity of a given covariate-setting using \emph{benchmark  popularity}:  the (weighted) average observed popularity amongst 100 nearest neighbors (in the feature-space) from the set of held-out articles (with weights based on inverse Euclidean distance).  Over our improvement set, the Pearson correlation between articles' observed popularity and  benchmark popularity is 0.28 (highly significant: $p \le 2\cdot 10^{-10}$).  This approach thus appears to be, on average, a reasonable way to benchmark performance (even though nearest-neighbor held-out articles can individually differ from the text of a particular pre/post-intervention article despite sharing similar values of our 29 measured covariates).  

Figure \ref{fig:writing} depicts the results of our personalized intervention for each article in our intervention set.  
The expected improvement produced by a particular intervention is estimated as the difference between the benchmark popularity of the post-intervention covariate-settings and the original covariate-settings of the article receiving the personalized intervention.  Table  \ref{tab:summarystat} summarizes these results.  A paired-sample $t$-test suggests our method is significantly superior on average ($p  < 2 \cdot 10^{-6}$).

\vspace*{5mm}
\begin{figure}[h!] \centering
\floatbox[{\capbeside \thisfloatsetup{capbesideposition={left,center},capbesidewidth=6cm,capbesidesep=quad}}]{figure}[\FBwidth]{\caption{Benchmark popularity changes produced by the personalized interventions for 300 articles suggested by our method with $\alpha = 0.05$ (Rejecting Uncertainty) vs. $\alpha = 0.5$ (Ignoring Uncertainty). The points (ie.\ articles) are colored according to the value of our personalized intervention objective with $\alpha = 0.05$. Using $\alpha=0.05$ outperforms $\alpha=0.5$ in this analysis in 177/300 articles in the improvement set.  
}
\label{fig:writing}
}{\includegraphics[width = 0.45\textwidth]{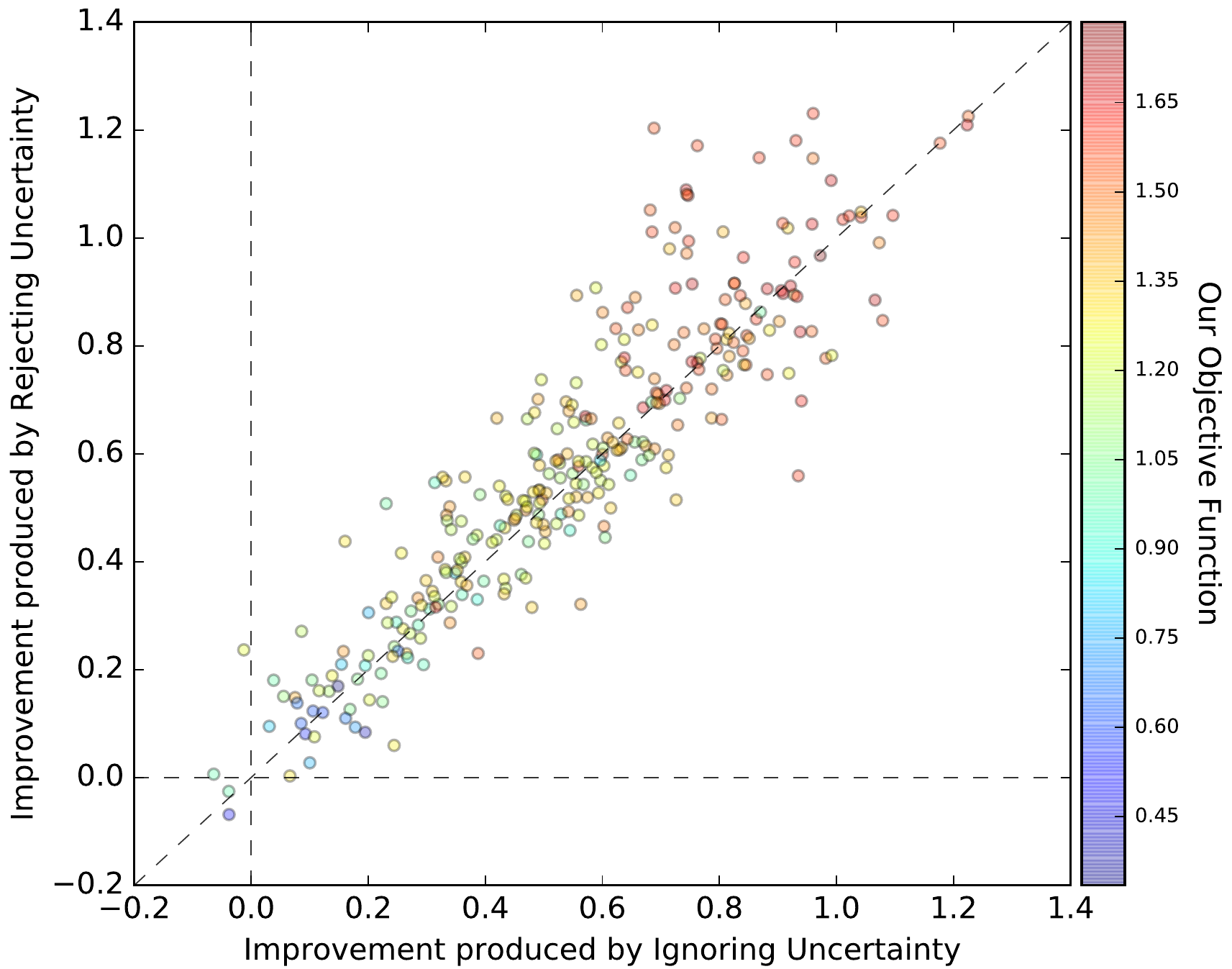}}
\end{figure} 

\vspace*{5mm}

\begin{table}[b!]  \begin{center}
\begin{tabular}{l c c c c}
\multicolumn{1}{c}{\bf \hspace*{5mm} Method} & \multicolumn{1}{c}{ \ \bf Mean \ } & \multicolumn{1}{c}{ \ \bf Median \ } & \multicolumn{1}{c}{\bf $\mathbf{0.05}^{\text{th}}$ Quantile} & \multicolumn{1}{c}{\bf Num.\ Negative}
\\ \hline \vspace*{-1mm} \\ 
Rejecting Uncertainty \ \ & 0.586 & 0.578 & 0.126 & 2 \vspace*{1mm}  \\
Ignoring Uncertainty \ \ & 0.552 & 0.555 & 0.105 & 4
\vspace*{1mm} \\
\hline
\end{tabular} \end{center}
\caption{Summary statistics for the benchmark popularity change produced by each method over the 300 articles of the intervention set. The last column counts the number of harmful interventions (with change $< 0$).}
\label{tab:summarystat}
\end{table}

To provide concrete examples, we present some  articles of the Business and Entertainment categories (taken from our improvement set).  For this business article:   \url{http://mashable.com/2014/07/30/how-to-beat-the-heat/}, our method proposes shifting the following 10 covariates  (see Table \ref{tab:textfeatures} for feature descriptions): 

num\_hrefs: +2, num\_self\_hrefs: -1.25, average\_token\_length: -1.771, kw\_avg\_min: +1.71, kw\_avg\_avg: +2, self\_reference\_min\_shares: +2, self\_reference\_max\_shares: +1.68, self\_reference\_avg\_sharess: +2, global\_subjectivity: +1.57,  global\_sentiment\_polarity: -2

For this entertainment article:  \url{http://mashable.com/2014/07/30/how-to-beat-the-heat/}, our method proposes  shifting the following 10 covariates:  

average\_token\_length: -1.55, kw\_avg\_min: + 1.63, kw\_avg\_avg: +2, self\_reference\_min\_shares: +2  self\_reference\_max\_shares: +1.85, self\_reference\_avg\_shares: +2.0, LDA\_00: +1.63, LDA\_01: -2, LDA\_04: +0.82, global\_subjectivity: +1.62

Indifferent to uncertainty, the method with $\alpha = 0.5$ advocates shifting all these covariates by the $\pm 2$ maximal allowed amounts, which leads to a 0.04 worse improvement in benchmark popularity compared with the covariate-changes specified above for this article.
\nopagebreak

\begin{table}[b!]  \scriptsize \begin{center} {
\begin{tabular}{l l}
\multicolumn{1}{l}{\bf \hspace*{5mm} Feature}  & \multicolumn{1}{l}{\bf Description}
\\ \hline \vspace*{-1mm} \\ 
n\_tokens\_title     & Number of words in the title   \vspace*{1mm}  \\     
n\_tokens\_content  &  Number of words in the content  \vspace*{1mm} \\
n\_unique\_tokens  &  Rate of unique words in the content \vspace*{1mm}  \\          
n\_non\_stop\_words &  Rate of non-stop words in the content  \vspace*{1mm} \\
 num\_hrefs   & Number of links \vspace*{1mm} \\ 
 num\_self\_hrefs       & Number of links to other articles published by Mashable  
 \vspace*{1mm} \\     
  average\_token\_length  &  Average length of the words in the content  
\vspace*{1mm} \\
 num\_keywords      &  Number of keywords in the metadata
 \vspace*{1mm} \\
data\_channel\_is\_lifestyle  &  Is the article category ``Lifestyle''? 
\vspace*{1mm} \\ 
data\_channel\_is\_entertainment & Is the article category ``Entertainment''? \vspace*{1mm} \\
data\_channel\_is\_bus & Is the article category ``Business''?   \vspace*{1mm} \\        data\_channel\_is\_socmed  & Is the article category ``Social Media''?  \vspace*{1mm} \\    
data\_channel\_is\_tech   & Is the article category ``Tech''? \vspace*{1mm} \\       data\_channel\_is\_world &  Is the article category ``World''?
\vspace*{1mm} \\    
kw\_avg\_min   &   Avg. shares of articles with the  least popular keyword used for this article
\vspace*{1mm} \\
kw\_avg\_max   & Avg. shares of articles with the most popular keyword used for this article  
\vspace*{1mm} \\            
kw\_avg\_avg &  Avg. shares of the average-popularity  keywords used for this article   \vspace*{1mm}  \\   
self\_reference\_min\_shares   & Min.\ shares of referenced articles in Mashable  \vspace*{1mm}  \\
self\_reference\_max\_shares  & Max.\ shares of referenced articles in Mashable  \vspace*{1mm}  \\
self\_reference\_avg\_shares   & Avg.\ shares of referenced articles in Mashable   \vspace*{1mm}  \\
LDA\_00 & Closeness to first LDA topic   \vspace*{1mm}  \\
LDA\_01 & Closeness to second LDA topic
\vspace*{1mm}  \\                      
LDA\_02  & Closeness to third LDA topic 
\vspace*{1mm}  \\                    
LDA\_03      &  Closeness to fourth LDA topic
\vspace*{1mm}  \\                 
LDA\_04      & Closeness to fifth LDA topic
\vspace*{1mm}  \\                 
global\_subjectivity   & Subjectivity score of the text
\vspace*{1mm}  \\       
global\_sentiment\_polarity  &  Sentiment polarity of the text
\vspace*{1mm}  \\   
title\_subjectivity    &  Subjectivity score of title  \vspace*{1mm}  \\       
title\_sentiment\_polarity  & Sentiment polarity of title   \vspace*{1mm}
\\
\hline
\end{tabular}} \end{center}
\caption{The 29 covariates of each article (dimensions of $X$ in this analysis).  Features involving the  share-counts of other articles and LDA were based only on data known before the publication date.}
\label{tab:textfeatures}
\end{table}

\clearpage
\section{Proofs and additional Theoretical Results}
\label{subsec:proofs}

\subsubsection*{Notation and Definitions}
\label{notation}
\addcontentsline{toc}{subsubsection}{\nameref{notation}}

All points $x \in \mathbb{R}^d$ lie in convex and compact domain $\mathcal{C} \subset \mathbb{R}^d$.

$C$ denotes constants whose value may change from line to line.

All occurrences of $f$ are implicitly referring to $f \mid \mathcal{D}_n$.

$\mu_n(\cdot)$, $\sigma^2_n(\cdot)$, and $\sigma_n(\cdot, \cdot)$ respectively denote the mean, variance, and covariance function of our posterior for $f \mid \mathcal{D}_n$ under the  GP$\big(0, k(x,x')\big)$ prior. 

$F^{-1}_Z(\alpha)$ denotes the $\alpha^{\text{th}}$ quantile of random variable $Z$. 

$\Phi^{-1}(\cdot)$ denotes the $N(0,1)$ quantile function.

$||\cdot||_k$ denotes the norm of reproducing kernel Hilbert space $\mathcal{H}_k$.

$\mathcal{B}_\delta(x) \subset \mathbb{R}^d$ denotes the ball of radius $\delta$ centered at $x \in \mathcal{C}$.

$\mathcal{I} \subseteq \{1,\dots,d\}$ represents the set of variables which are intervened upon in sparse settings.

$\text{pa}(Y)$ denotes the set of variables which are parents of $Y$ in a causal \emph{directed acyclic graph} (DAG) \citepsi{Pearl2000si}

$\text{desc}(\mathcal{I})$ is the set of variables which  are descendants of at least one variable in $\mathcal{I}$ according to the causal DAG.

$A^C$ denotes the complement of set $A$.

The \emph{squared exponential} kernel (with length-scale parameter $l > 0$) is defined: $$k(x,x') = \exp\Big( -\frac{1}{2 l^2} ||x - x' ||^2 \Big) $$

The \emph{Mat\'ern} kernel (with another parameter $\nu > 0$  controlling smoothness of sample paths) is defined: $$k(x,x') = \frac{2^{1-\nu}}{\Gamma(\nu)} r^\nu B_\nu(r) \ \ \ \text{ where } \ r = \frac{\sqrt{2\nu}}{l} || x - x'||, B_\nu \text{ is a modified Bessel function}$$

Random variables $\varepsilon^{(1)}, \dots, \varepsilon^{(n)}$ form a \emph{martingale difference sequence} which is \emph{uniformly bounded} by $\sigma$ if 
$ \ \mathbb{E}[\varepsilon^{(i)} \mid \varepsilon^{(i-1)},\dots, \varepsilon^{(1)}] = 0 $ and $\varepsilon^{(i)} \le \sigma \ \  \forall i \in \mathbb{N}$.

A function $f$ is \emph{Lipshitz continuous} with constant $L$ if: $|f(x) - f(x')| \le L |x - x'|$ for every $x, x' \in \mathcal{C}$.

Suppose $\rho > 0$ is expressed as $\rho = m + \eta$ for nonnegative integer $m$ and  $0 < \eta \le 1$. \\
The \emph{H\"{o}lder space} $C^\rho[0,1]^d$ is the space of  functions with existing partial derivatives of orders $(k_1, \dots, k_d)$ for all integers $k_1, \dots, k_d \ge 0$ satisfying $k_1 + \dots +  k_d \le m$. Additionally, each function's highest order partial derivative must form a function $h$ that satisfies: $|h(x) - h(y)| \le C |x-y|^\eta$ for any $x,y$.

\bigskip

\noindent \begin{thm}[\citesi{VanderVaart2011si}] \label{thm:van}
Under the assumptions of Theorem \ref{thm:close}:
$$\mathbb{E}_{\mathcal{D}_n} \int \int_\mathcal{C} [f(x) - f^*(x)]^2 p_X(x) \mathrm{d}x \ \mathrm{d} \Pi_n ( f \mid \mathcal{D}_n) \ \le \ C \cdot \Psi_{\hspace*{-0.2mm}f^*}(n)$$ 
\end{thm}
where $\Psi_{\hspace*{-0.2mm}f^*}^{-1}(n)$ is defined as in \S\ref{sec:theory}.  See \citesi{VanderVaart2011si} for a detailed discussion about this function.

\subsubsection*{Proof of Theorem \ref{thm:close}}
\label{proof:close}
\addcontentsline{toc}{subsubsection}{\nameref{proof:close}}

\begin{proof} Recall $G_{x}(T) := f(T(x)) - f(x) \mid \mathcal{D}_n$ depends on $f$. We fix $x_0, T(x_0) \in \mathcal{C}$ and adapt the bound provided by Theorem \ref{thm:van} to show our result.  Let $\mathcal{B}_\delta(x) \subset \mathcal{C}$ denote the ball of radius $0 < \delta < \frac{1}{2}$ centered at $x \in \mathcal{C}$.
We first establish the bound:
\begin{align*}
& \int_\mathcal{C} \big|f(x) - f^*(x) \big| p_X(x) \ \mathrm{d}x 
\\
\ge & \int_{\mathcal{B}_\delta(x_0)} \big| f(x) - f^*(x) \big| p_X(x) \ \mathrm{d}x  + \int_{\mathcal{B}_\delta(T(x_0))} \big| f(x) - f^*(x) \big| p_X(x) \ \mathrm{d}x        
\\
\ge & a \cdot \text{Vol}(\mathcal{B}_\delta) \Big[ \min_{x \in \mathcal{B}_\delta(x_0)} \big| f(x) - f^*(x) \big| + \min_{x \in \mathcal{B}_\delta(T(x_0))} \big| f(x) - f^*(x) \big| \Big]
\\
\ge &  a \cdot \text{Vol}(\mathcal{B}_\delta) \cdot \Big[ \Big|  f(T(x_0)) - f(x_0) - \big[ f^*(T(x_0)) - f^*(x_0) \big] \Big| - 8 \delta L \Big]  \\
\ge & a \cdot \text{Vol}(\mathcal{B}_\delta) \cdot \Big[  \Big| G_{x_0}(T) - G^*_{x_0}(T) \Big| - 8 \delta L \Big] \numberthis \label{eq:holder}
\end{align*}

where Vol$(\mathcal{B}_\delta) = \mathcal{O}( \delta^d)$.  Theorem \ref{thm:van} implies the following inequality (ignoring constant factors):

\begin{align*}
[C \cdot &  \Psi_{\hspace*{-0.2mm}f^*}(n)   ]^{1/2} 
\\
\ge &  \Bigg[ \mathbb{E}_{\mathcal{D}_n} \int \int_\mathcal{C} [f(x) - f^*(x)]^2 p_X(x) \ \mathrm{d}x \ \mathrm{d} \Pi_n ( f \mid \mathcal{D}_n)  \Bigg]^{1/2}
\\
\ge  &  \mathbb{E}_{\mathcal{D}_n}  \int  \int_\mathcal{C} \big| f(x) - f^*(x) \big| p_X(x) \ \mathrm{d}x  \ \mathrm{d} \Pi_n ( f \mid \mathcal{D}_n)  
\tag*{by Jensen's inequality} \\
\ge & a \delta^d \cdot  \mathbb{E}_{\mathcal{D}_n}  \int \big| G_{x_0}(T) - G^*_{x_0}(T)  \big| - \delta L \ \ \mathrm{d} \Pi_n ( f \mid \mathcal{D}_n) 
\tag*{via the bound from (\ref{eq:holder})}
\\
 = & - a L \delta^{d+1} + a \delta^d \cdot  \mathbb{E}_{\mathcal{D}_n} \int_0^\infty \Pr \Big(  \big| G_{x_0}(T) - G^*_{x_0}(T) \big| \ge r  \Big) \ \mathrm{d}r 
 \\
 = & - a L \delta^{d+1} + a \delta^d \cdot \mathbb{E}_{\mathcal{D}_n} \int_0^1 F^{-1}_{ | G_{x_0}(T) - G^*_{x_0}(T) | }(\widetilde{\alpha})  \ \  \mathrm{d}\widetilde{\alpha} 
 \\
 \ge & - a L \delta^{d+1} + a \delta^d \cdot \mathbb{E}_{\mathcal{D}_n} \int_\alpha^1  F^{-1}_{G_{x_0}(T)}(\widetilde{\alpha}) - G^*_{x_0}(T) \ \  \mathrm{d}\widetilde{\alpha} 
\\
\ge &  - a L \delta^{d+1} + a (1-\alpha) \delta^d \cdot \mathbb{E}_{\mathcal{D}_n} \Big[ F^{-1}_{G_{x_0}(T)}(\alpha) - G^*_{x_0}(T) \Big] 
\numberthis \label{eq:onewaybound}
\end{align*}

We can similarly bound $G^*_{x_0}(T) - F^{-1}_{G_{x_0}(T)}(\alpha)$:
\begin{align*}
& - a L \delta^{d+1} + a \delta^d \cdot \mathbb{E}_{\mathcal{D}_n} \int_0^1 F^{-1}_{ | G^*_{x_0}(T) - G_{x_0}(T)  |}(\widetilde{\alpha})  \ \mathrm{d}\widetilde{\alpha}  
\\
\ge & - a L \delta^{d+1} + a \delta^d \cdot  \mathbb{E}_{\mathcal{D}_n} \int_0^\alpha  G^*_{x_0}(T) -  F^{-1}_{G_{x_0}(T) }(\widetilde{\alpha}) \  \ \mathrm{d}\widetilde{\alpha}
\\
\ge & - a L \delta^{d+1} + a \alpha  \delta^d \cdot \mathbb{E}_{\mathcal{D}_n} \Big[ G^*_{x_0}(T) -  F^{-1}_{G_{x_0}(T) }(\alpha) \Big]
\numberthis \label{eq:otherwaybound}
\end{align*}

Choosing $\delta := [\Psi_{\hspace*{-0.2mm}f^*}(n)]^{\frac{1}{2(d+1)}}$ and  combining (\ref{eq:onewaybound}) and (\ref{eq:otherwaybound}) produces the desired result, since assuming $\alpha < 0.5$ implies $\alpha < 1 - \alpha$.  \end{proof}

\subsubsection*{Proof of Theorem \ref{thm:popclose}}
\label{proof:popclose}
\addcontentsline{toc}{subsubsection}{\nameref{proof:popclose}}

\begin{proof} Combining the results of Lemmas  \ref{lem:triangle1} and \ref{lem:triangle2} below, we  obtain the desired upper bound through a straightforward application of the triangle inequality.  Note that we've simplified the bound using the identity $-\log(1-\alpha) < 1/\alpha$ for $\alpha < 0.5$. \end{proof}

\bigskip
\begin{lem} \label{lem:triangle1} Under the assumptions of Theorem \ref{thm:popclose}, for any $x, T(x) \in \mathcal{C}$:
$$ \mathbb{E}_{\mathcal{D}_n} \ \Big|F^{-1}_{G_n(T)}(\alpha)  - F^{-1}_{G_X(T)}(\alpha) \Big| 
\le  
C \cdot \Big[ \frac{-L^2d}{n} \log(1-\alpha) \Big]^{1/2} 
$$
\end{lem}
\begin{proof}[Proof of Lemma \ref{lem:triangle1}]
Define random variables $Z_i := f(T(x^{(i)}) - f(x^{(i)}) \mid \mathcal{D}_n$ for $i=1,\dots, n$. \\
Note that these variables all share the same expectation: $\mathbb{E}_X[Z] := \mathbb{E}_X[Z_i] = G_X(T)$ and $G_n(T) = \frac{1}{n}\sum_{i=1}^n Z_i$.  
The Lipschitz continuity of $f$ combined with the fact that $\mathcal{C} = [0,1]^d$ implies: $Z_i \in [-L \sqrt{d}, L \sqrt{d}]$ for all $i$.  
Thus, Hoeffding's inequality ensures: 
\begin{align*}
& \Pr \Bigg(\Bigg| G_n(T)   -  G_X(T) \Bigg| \ge t \Bigg) \le 2 \exp\Bigg( \frac{-n t^2}{2 L^2 d} \Bigg) 
\\
\Rightarrow \ & F^{-1}_{\big| G_n(T) - G_X(T) \big|}(\alpha) 
\le 
C \cdot \Big[ \frac{-L^2d}{n} \log(1-\alpha) \Big]^{1/2} 
\\
\end{align*}
Because posteriors $ G_n(T), G_X(T)$ follow a  Gaussian distribution:
\begin{align*}
& F^{-1}_{G_n(T)}(\alpha) - F^{-1}_{G_X(T)}(\alpha)  
\le 
F^{-1}_{\big| G_n(T) - G_X(T) \big|}(\alpha) 
\\
\text{ and } & F^{-1}_{G_X(T)}(\alpha)  -  F^{-1}_{G_n(T)}(\alpha) 
\le
F^{-1}_{\big| G_n(T) - G_X(T) \big|}(\alpha) 
\end{align*}
\end{proof}

\bigskip
\begin{lem} \label{lem:triangle2} 
Under the assumptions of Theorem \ref{thm:popclose}, for any $x, T(x) \in \mathcal{C}$:
$$ \mathbb{E}_{\mathcal{D}_n} \ \Big|F^{-1}_{G_X(T)}(\alpha)  - G_X^*(T) \Big| 
\le 
\frac{C}{\alpha} \cdot \Big(L + \frac{1}{a} \Big) \cdot  [\Psi_{f^*}(n)]^{1 / [2(d+1)]} 
$$
\end{lem}
\begin{proof}[Proof of Lemma \ref{lem:triangle2}]
A similar argument as the proof of Theorem \ref{thm:close} applies here.  We again first bound:
\begin{align*}
&  \int_\mathcal{C} \big|f(x) - f^*(x) \big| p_X(x) \ \mathrm{d}x 
\\
\ge & a \cdot \text{Vol}(\mathcal{B}_\delta) \cdot \Bigg[ \int_\mathcal{C} \big|f(x) - f^*(x) \big| p_X(x) \ \mathrm{d}x +  \int_\mathcal{C} \big| f(T(x))- f^*(T(x)) \big| p_X(x) \ \mathrm{d}x - 8 \delta L \Bigg]
\\
\ge & a \cdot \text{Vol}(\mathcal{B}_\delta) \cdot \Bigg[  \Big| \mathbb{E}_X [f(x) - f^*(x)] + \mathbb{E}_X [f(T(x)) - f^*(T(x))] \Big| - 8 \delta L \Bigg] 
\end{align*}
Following the same reasoning as in the proof of Theorem \ref{thm:close}, we obtain (up to constant factors):
\begin{align*}
- a L \delta^{d+1} + a \alpha  \delta^d \cdot \mathbb{E}_{\mathcal{D}_n} \Big[  G_X^*(T) -  F^{-1}_{G_X(T)}(\alpha) \Big] \le  [C \cdot   \Psi_{\hspace*{-0.2mm}f^*}(n)   ]^{1/2} 
\end{align*}
and we can use the same argument to similarly bound 
$$\mathbb{E}_{\mathcal{D}_n} \Big[ F^{-1}_{G_X(T)}(\alpha) - G_X^*(T) \Big]$$  \end{proof}

\subsubsection*{Proof of Theorem \ref{thm:dooperation}}
\label{proof:dooperation}
\addcontentsline{toc}{subsubsection}{\nameref{proof:dooperation}}

Here, we employ subscripts to index particular covariates of $X$.  The notation $[a_{R}, a_{S}] = a \in \mathbb{R}^d$ is used to denote a vector assembled from disjoint subsets of dimensions $R,S \subseteq \{1,\dots, d\}$.  Regardless of the ordering of these partitions in our notation, we assume they are correctly arranged in the assembled vector based on  their subscript-indices (ie.\ $a = [a_{R}, a_{S}] = [a_{S}, a_{R}]$).

\begin{proof}
\begin{align*}
& \mathbb{E}_{\text{do}(X_{\mathcal{I}} =  z_{\mathcal{I}})} \big[ f^*(x) \big] 
\\
 = & \int f^*\big([x_{\mathcal{I}^C} ,  z_{\mathcal{I}}]\big) \  p\big(x_{\mathcal{I}^C} \mid do(X_\mathcal{I} = z_\mathcal{I}) \big) \ \mathrm{d}x_{\mathcal{I}^C}
 \\
 = & \int \int 
 f^*\big([x_{\text{pa}(Y) \setminus \mathcal{I}} ,  z_{\mathcal{I} \cap \text{pa}(Y)} ,  a_{\mathcal{I}^C \setminus \text{pa}(Y)}]\big) \cdot
p\big( 
x_{\mathcal{I}^C \setminus \text{pa}(Y)} 
\mid x_{\text{pa}(Y) \setminus \mathcal{I}} ,
do(X_\mathcal{I} = z_\mathcal{I}) 
\big)
\\
& \hspace*{15mm}
 \cdot   p\big(
 x_{\text{pa}(Y) \setminus \mathcal{I}}   \mid do(X_\mathcal{I} = z_\mathcal{I}) 
 \big) \ 
 \mathrm{d}x_{\mathcal{I}^C \setminus \text{pa}(Y)}  \ 
 \mathrm{d}x_{\text{pa}(Y) \setminus \mathcal{I}}
 \\
 & \text{ where covariate-subset } a_{\mathcal{I}^C \setminus \text{pa}(Y)} \text{ can take arbitrary values since $f^*$ is constant along covariates  $\notin \text{pa}(Y)$} 
 \\
 = & \int f^*\big([x_{\text{pa}(Y) \setminus \mathcal{I}} ,  z_{\mathcal{I} \cap \text{pa}(Y)} ,  a_{\mathcal{I}^C \setminus \text{pa}(Y)}]\big) \
 p\big(
 x_{\text{pa}(Y) \setminus \mathcal{I}}   \mid do(X_\mathcal{I} = z_\mathcal{I}) 
 \big) \ 
 \mathrm{d}x_{\text{pa}(Y) \setminus \mathcal{I}}
 \\
 = & \int f^*\big([x_{\text{pa}(Y) \setminus \mathcal{I}} ,  z_{\mathcal{I} \cap \text{pa}(Y)} ,  a_{\mathcal{I}^C \setminus \text{pa}(Y)}]\big) \
 p\big(x_{\text{pa}(Y) \setminus \mathcal{I}} \big) 
 \ 
 \mathrm{d}x_{\text{pa}(Y) \setminus \mathcal{I}}
 \\
 & \text{ since the marginal distribution over }  X_{\text{pa}(Y) \setminus \mathcal{I}} \text{ equals the \emph{do}-distribution by assumption (A\ref{as:bestdo})} 
 \\  
 = & \int \int 
 f^*\big([x_{\text{pa}(Y) \setminus \mathcal{I}} ,  z_{\mathcal{I} \cap \text{pa}(Y)} ,  x_{\mathcal{I}^C \setminus \text{pa}(Y)}]\big) \
p\big( 
x_{\mathcal{I}^C \setminus \text{pa}(Y)} 
\mid x_{\text{pa}(Y) \setminus \mathcal{I}}
\big)
\
p\big( x_{\text{pa}(Y) \setminus \mathcal{I}} \big) \
\mathrm{d}x_{\mathcal{I}^C \setminus \text{pa}(Y)} \
\mathrm{d}x_{\text{pa}(Y) \setminus \mathcal{I}}
 \\
 = & \mathbb{E}_X \Big[ f^*(T_{\mathcal{I}\shortrightarrow z}(x)) \Big]
\end{align*} \end{proof}

\subsubsection*{Proof of Theorem \ref{thm:bestdofound}}
\label{proof:bestdofound}
\addcontentsline{toc}{subsubsection}{\nameref{proof:bestdofound}}

Recall we defined: 
\begin{equation}
\mathcal{I^*} := \argmin \Big\{ |\mathcal{I}'| \ \text{ s.t. } \ \exists \ T_{\mathcal{I}'\shortrightarrow z} \in  \argmax_{T_{\mathcal{I}\shortrightarrow z} : |\mathcal{I}| \le k} \mathbb{E}_X \big[ f^*(T_{\mathcal{I}\shortrightarrow z}(x)) - f^*(x) \big]  \Big\}
\end{equation}
as the intervention set corresponding to the optimal sparse covariate-fixing transformation (taken to be the set of minimal cardinality in cases with multiple maxima).

\begin{proof}  Since $\mathbb{E}_X [f^*(T_{\mathcal{I}\shortrightarrow z}(x))]$ does not change when $z_j := [T_{\mathcal{I}\shortrightarrow z}(x)]_j$ is altered for any $j \notin \text{pa}(Y)$, including variables outside of the parent set in $\mathcal{I}$ does not improve this quantity.  Thus, either $ \text{pa}(Y) \subseteq \mathcal{I}^*$, or $\mathcal{I}^* \subset \text{pa}(Y)$.  The first case immediately implies (A\ref{as:bestdo}).  When $\mathcal{I}^* \subset \text{pa}(Y)$: our assumption that no variable in $\text{pa}(Y)$ is a descendant of other parents implies the other parents must belong the complement of  $\text{desc}(\mathcal{I}^*)$, since this is a subset of $\text{desc}\big(\text{pa}(Y)\big)$.
\end{proof}

\subsection*{Theorem \ref{thm:notbad} and Proof}
\label{supsec:thmnotbad}
\addcontentsline{toc}{subsubsection}{\nameref{supsec:thmnotbad}}

\begin{thm} \label{thm:notbad}
Suppose we adopt a GP$\big(0, k(x,x')\big)$ prior and, in addition to the assumptions outlined in \S\ref{sec:theory}, the following conditions hold: \refstepcounter{assumption} (A\theassumption) \label{as:srinivas1} $f^* \in \mathcal{H}_k(\mathcal{C})$ which is the RKHS induced by our covariance function $k$ with norm $||\cdot||_k$ (cf. \citesi{Rasmussen2006si} \S6.1),
\refstepcounter{assumption}
(A\theassumption) \label{as:srinivas2} noise variables $\varepsilon^{(i)}$ form a uniformly bounded martingale difference sequence $\varepsilon^{(i)} \le \sigma$ for $i= 1,\dots, n$.

Then, for any $\displaystyle x, T(x) \in \mathcal{C} : \ \  F^{-1}_{G_x(T)}(\alpha) \le G^*_x(T)  $
$$\text{with probability (over the noise) greater than } \   1 - C(n+1) \cdot \exp\left( -\frac{[\Phi^{-1}(\alpha)]^2 - 2 ||f^*||^2_k}{\gamma_n} \right)$$
\end{thm}

 In Theorem \ref{thm:notbad}, $\displaystyle \gamma_n := \max_{A \subset \mathcal{C} : |A| = n} \frac{1}{2}\log \big| \mathbf{I} + \sigma^{-2} \mathbf{K}_A \big|$ is a kernel-dependent quantity ($\mathbf{K}_A := [k(x,x')]_{x,x' \in A}$) which, in the Gaussian setting, is the mutual information between $f$ and observations of $Y$ at the most informative choice of $n$ points.  When the kernel satisfies $k(x, x') \le 1$, the following bounds are known \citepsi{Srinivas2010si}: $\gamma_n = \mathcal{O}(d \log n)$ for the linear kernel, $\gamma_n = \mathcal{O}((\log n)^{d+1})$ for the squared exponential kernel, and $\gamma_n = \mathcal{O}(n^{d(d+1)/(2\nu + d(d+1))} (\log n))$ for the Mat\'ern kernel with smoothness parameter $\nu$.

Note that while $f^*$ is not required to be drawn from our prior and $\varepsilon$ may be non-Gaussian, this result assumes the kernel $k$ and noise-level $\sigma$ are correctly set.  Our proof relies on the following statement:

\bigskip

\noindent \begin{thm}[\citesi{Srinivas2010si}]
\label{thm:srinivas}
Assume conditions (A\ref{as:srinivas1}) - (A\ref{as:srinivas2}), fix $\delta \in (0,1)$, and define:
$$\displaystyle \beta_n := 2 || f^* ||^2_k + 300 \gamma_n [\log (n / \delta)]^3 $$
Then: \hfill $
\begin{aligned}[t]
\Pr\Big[ \forall x \in \mathcal{C} : \ | \mu_n (x) - f^*(x) | \le \sqrt{\beta_{n+1}} \sigma_n(x) \Big] \ge 1- \delta 
\end{aligned}$\hfill\null
\end{thm}

\bigskip

\begin{proof}[Proof of Theorem \ref{thm:notbad}] Fix $x, T(x) \in \mathcal{C}$, and define $\displaystyle \delta := (n+1) \cdot \exp\left( -\frac{[\Phi^{-1}(\alpha)]^2 - 2 ||f^*||^2_k}{300 \gamma_n} \right)$. \\
In this case, $-\sqrt{\beta_{n+1}} = \Phi^{-1}(\alpha)$ (see definition in previous theorem).

Theorem \ref{thm:srinivas} implies that with probability $\ge 1 - \delta$: \\
\ $ | \mu_n (x) - f^*(x) | \le -\Phi^{-1}(\alpha) \cdot \sigma_n(x) $ \ and \ $| \mu_n (T(x)) - f^*(T(x)) | \le -\Phi^{-1}(\alpha) \cdot  \sigma_n(T(x)) $ \\

Since our posterior is Gaussian: 
$$ F^{-1}_{G_x(T)} (\alpha) = \mu_n(T(x)) - \mu_n(x) + \Phi^{-1}(\alpha) \bigg[\sigma^2_n(T(x)) +  \sigma^2_n(x) - 2 \sigma_n(x, T(x)) \bigg]^{1/2}$$ 
Therefore: 
\begin{align*}
& f^*(T(x)) - f^*(x) - F^{-1}_{G_x(T)} (\alpha)
\\
= & f^*(T(x)) - \mu_n(T(x)) + \mu_n(x) - f^*(x) - \Phi^{-1}(\alpha) \bigg[\sigma^2_n(T(x)) +  \sigma^2_n(x) - 2 \sigma_n(x, T(x)) \bigg]^{1/2} \\
\le &  f^*(T(x)) - \mu_n(T(x)) + \mu_n(x) - f^*(x) - \Phi^{-1}(\alpha) \bigg[\sigma^2_n(T(x)) +  \sigma^2_n(x) + 2 \sqrt{ \sigma^2_n(x) \sigma^2_n(T(x))} \bigg]^{1/2} 
\\
& \tag*{since we assume $\alpha \le 0.5 \Rightarrow \Phi^{-1}(\alpha) \le 0$, and  \ $\big| \sigma_n(x, T(x)) \big| \le \sqrt{ \sigma^2_n(x) \sigma^2_n(T(x))}$} \\
= &  f^*(T(x)) - \mu_n(T(x)) + \mu_n(x) - f^*(x) - \Phi^{-1}(\alpha) \bigg[\sigma_n(T(x)) +  \sigma_n(x) \bigg] 
\\
= &  \big[f^*(T(x)) - \mu_n(T(x)) - \Phi^{-1}(\alpha) \sigma_n(T(x)) \big]  + \big[ \mu_n(x) - f^*(x) - \Phi^{-1}(\alpha) \sigma_n(x) \big] 
\end{align*}
which is less than 0 with probability at most $\delta$.
\end{proof}

\clearpage \newpage
\subsection*{Additional References for the Supplementary Material}
\bibliographystylesi{agsm}
{\bibliographysi{InterventionOptBibliography}}

\end{document}